\documentclass{article}

\usepackage{microtype}
\usepackage[pdftex]{graphicx}
\usepackage{subcaption}
\usepackage{booktabs} %
\usepackage{xcolor}

\usepackage[pagebackref=true,breaklinks=true,colorlinks,bookmarks=false]{hyperref}
 \hypersetup{
 linkcolor=red,
 filecolor=red,
 citecolor=green,      
 urlcolor=cyan,
}

\usepackage{pgfplots}
\pgfplotsset{compat = newest}
\usepackage{multirow}

\usepackage{framed}
\usepackage{tcolorbox}
\usepackage{subcaption}

\usepackage{amsmath}
\usepackage{amssymb}
\usepackage{mathrsfs}
\usepackage{mathtools}
\usepackage{amsthm}
\usepackage[ruled,vlined]{algorithm2e}
\usepackage{algorithmic}
\usepackage{listings}
\usepackage{makecell}
\usepackage{colortbl}
\usepackage{color}
\usepackage{wrapfig}
\usepackage{cancel}
\usepackage{soul,xcolor}
\usepackage{pifont}
\usepackage{tikz}
\usetikzlibrary{shapes, positioning, arrows.meta, calc, decorations.pathmorphing, quotes}

\usepackage[capitalize,noabbrev]{cleveref}

\theoremstyle{plain}
\newtheorem{theorem}{Theorem}[section]

\newtheorem{lemma}[theorem]{Lemma}
\newtheorem{corollary}[theorem]{Corollary}
\theoremstyle{definition}
\newtheorem{definition}[theorem]{Definition}

\theoremstyle{remark}

\newcommand{\alink}[1]{\href{#1}{paper-link}}

\usepackage{enumitem}
\usepackage{ amssymb }

\definecolor{citecolor}{HTML}{0071BC}
\definecolor{linkcolor}{HTML}{ED1C24}
\hypersetup{colorlinks=true, linkcolor=linkcolor, citecolor=citecolor,urlcolor=black}

\definecolor{commentcolor}{RGB}{110,154,155}   

\usepackage{amsmath,amsfonts,bm}

\def\eqref#1{equation~\ref{#1}}

\def\1{\bm{1}}

\DeclareMathAlphabet{\mathsfit}{\encodingdefault}{\sfdefault}{m}{sl}
\SetMathAlphabet{\mathsfit}{bold}{\encodingdefault}{\sfdefault}{bx}{n}

\DeclareMathOperator*{\argmin}{arg\,min}

\usepackage{epstopdf}
\usepackage{microtype}
\usepackage{graphicx}
\usepackage{booktabs} %

\usepackage{hyperref}

\usepackage[accepted]{icml2024}

\usepackage{amsmath}
\usepackage{amssymb}
\usepackage{mathtools}
\usepackage{amsthm}

\usepackage[capitalize,noabbrev]{cleveref}

\usepackage[textsize=tiny]{todonotes}

\icmltitlerunning{Unveiling the Dynamics of Information Interplay in Supervised Learning}

\begin{document}

\twocolumn[
\icmltitle{Unveiling the Dynamics of Information Interplay in Supervised Learning}

\icmlsetsymbol{equal}{*}

\begin{icmlauthorlist}
\icmlauthor{Kun Song}{equal,ustb}
\icmlauthor{Zhiquan Tan}{equal,thu}
\icmlauthor{Bochao Zou}{ustb}
\icmlauthor{Huimin Ma \textsuperscript{\dag}}{ustb}
\icmlauthor{Weiran Huang \textsuperscript{\dag}}{sjtu,ailab}
\end{icmlauthorlist}

\icmlaffiliation{ustb}{University of Science and Technology Beijing}
\icmlaffiliation{thu}{Department of Mathematical Sciences, Tsinghua University}
\icmlaffiliation{sjtu}{MIFA Lab, Qing Yuan Research Institute, SEIEE, Shanghai Jiao Tong University}
\icmlaffiliation{ailab}{Shanghai AI Laboratory}

\icmlcorrespondingauthor{Weiran Huang}{weiran.huang@outlook.com}
\icmlcorrespondingauthor{Huimin Ma}{mhmpub@ustb.edu.cn}

\icmlkeywords{matrix information theory, supervised learning, semi-supervised learning, Nerual Collapse}

\vskip 0.3in
]

\printAffiliationsAndNotice{\icmlEqualContribution} %

\begin{abstract}

In this paper, we use matrix information theory as an analytical tool to analyze the dynamics of the information interplay between data representations and classification head vectors in the supervised learning process. Specifically, inspired by the theory of Neural Collapse, we introduce matrix mutual information ratio (MIR) and matrix entropy difference ratio (HDR) to assess the interactions of data representation and class classification heads in supervised learning, and we determine the theoretical optimal values for MIR and HDR when Neural Collapse happens. Our experiments show that MIR and HDR can effectively explain many phenomena occurring in neural networks, for example, the standard supervised training dynamics, linear mode connectivity, and the performance of label smoothing and pruning. Additionally, we use MIR and HDR to gain insights into the dynamics of grokking, which is an intriguing phenomenon observed in supervised training, where the model demonstrates generalization capabilities long after it has learned to fit the training data. Furthermore, we introduce MIR and HDR as loss terms in supervised and semi-supervised learning to optimize the information interactions among samples and classification heads. The empirical results provide evidence of the method's effectiveness, demonstrating that the utilization of MIR and HDR not only aids in comprehending the dynamics throughout the training process but can also enhances the training procedure itself.

\end{abstract}

\section{Introduction}

Supervised learning is a significant part of machine learning, tracing its development back to the early days of artificial intelligence. Leveraging ample annotated data from large-scale datasets like ImageNet \cite{krizhevsky2012imagenet} and COCO \cite{lin2014microsoft}, supervised learning has achieved outstanding performance in tasks such as image recognition \cite{he2016deep, girshick2015fast, ronneberger2015u}, speech recognition \cite{hinton2012deep, chan2016listen}, and natural language processing \cite{vaswani2017attention}, thereby advancing the development of artificial intelligence. Concurrently, with its enhanced performance in real-world applications, some interesting phenomena in supervised learning, such as Neural Collapse~\cite{papyan2020prevalence}, linear mode connectivity~\cite{frankle2020linear}, and grokking \cite{power2022grokking} have emerged. More and more work is beginning to explore the reasons behind these phenomena.

Neural Collapse \cite{papyan2020prevalence} is an interesting phenomenon observed during the training process of supervised learning. In different stages of network training, features of samples within the same class become more similar in the feature space, meaning that intra-class differences decrease. At the same time, feature vectors of different classes become more distinct in feature space, leading to more significant inter-class differences. In supervised learning classification tasks, after prolonged training, a special alignment occurs, this alignment is formed between the weights of the network's final fully connected layer and the feature vectors of the classes. This indicates that for each class, the centroid of its feature vector almost coincides with the weight vector of its corresponding classifier (classification head vector).

Existing work on the theory of Neural Collapse primarily solely focuses on feature or classification head similarity, with few studies exploring the information between features and class classification heads. We introduce matrix information theory as an analytical tool, using similarity matrices constructed from sample features and class classification heads to analyze the information dynamics along the training process.
According to Neural Collapse, for a well-trained model at its terminal phase, sample features and the corresponding class classification heads align well. Thus, at the stage of Neural Collapse, the similarity matrix of sample features aligns with the similarity matrix constructed by the corresponding class classification heads. Therefore, we first theoretically calculate their matrix mutual information ratio and the matrix entropy difference ratio at the point of Neural Collapse. We find that at the point of Neural Collapse, the theoretical MIR is nearing its maximum, and the HDR reaches its theoretical minimum. Therefore, we expect an increase in MIR and a decrease in HDR during the training process. Experiments demonstrate that MIR and HDR effectively describe such phenomena during training. Motivated by the success of understanding dynamics in standard training, we also explore their uses in other settings like linear mode connectivity, label smoothing, pruning, and grokking. Compared to accuracy, MIR and HDR not only can describe the above phenomena but also possess unique analytical significance. We also explore integrating information constraints on model performance, by adding MIR and HDR as a loss term in supervised and semi-supervised learning. Experiments show that adding additional information constraints during training effectively enhances model performance, especially in semi-supervised learning with limited labeled samples, where information constraints help the model better learn information from unlabeled samples.

Our contributions are as follows:

1. Motivated by Neural Collapse and matrix information theory, we introduce two new metrics: Matrix Mutual Information Ratio (MIR) and Matrix Entropy Difference Ratio (HDR), for which we also deduce their theoretical values when Neural Collapse happens.

2. Through rigorous experiments, we find that MIR and HDR are capable of explaining various phenomena, such as the standard training of supervised learning, linear mode connectivity, pruning, label smoothing, and grokking. 

3. We integrate matrix mutual information and information entropy differences as a loss term in both supervised and semi-supervised learning. Experiments demonstrate that these information metrics can effectively improve model performance.

\section{Related Work}

\paragraph{Neural network training phenomenon.}

Recent research has revealed several interesting phenomena that are significant for understanding the behavior and learning dynamics of neural networks. Firstly, \citet{papyan2020prevalence} observe that in the final stages of deep neural network training, the feature vectors of the last layer tend to converge to their class centroids, and these class centroids align with the weights of the corresponding class in the final fully connected layer. This phenomenon is known as \textbf{Neural Collapse}. Neural Collapse occurs in both MSE loss and cross-entropy loss \cite{han2021neural, zhou2022all}. Secondly, \citet{frankle2020linear} find that models trained from the same starting point, even when changing the input data sequence and data augmentation, eventually converge to the same local area. This phenomenon is termed \textbf{Linear Mode Connectivity}, which is influenced by architecture, training strategy, and dataset \cite{altintacs2023disentangling}. Lastly, \citet{power2022grokking} discover that after prolonged training, models can transition from simply memorizing data to inductively processing it. This phenomenon is known as the \textbf{Grokking}. \citet{nanda2022progress} finds the connections of grokking on the modulo addition task with trigonometric functions.

\paragraph{Information theory.}

Traditional information theory provides a universally applicable set of fundamental concepts and metrics to understand the relationship between probability distributions and information \cite{wang2021adaptive}. However, when dealing with high-dimensional data and complex data structures, traditional information theory tools struggle to analyze higher-order relationships within the data. As an extension and advancement of traditional information theory, matrix information theory broadens the scope of information theory to encompass the analysis of inter-matrix relationships. This enables a better understanding of the latent structures in data and more effective handling of complex relationships in high-dimensional data \cite{bach2022information}. There have been works that utilize matrix mutual information to analyze neural networks. \citet{tan2023information} use matrix mutual information to study the Siamese architecture self-supervised learning methods. \citet{zhang2023matrix} point out the relationship between effective rank, matrix entropy, and equiangular tight frame.

\paragraph{Semi-supervised learning.}
Semi-supervised learning focuses on how to train a better model using a small number of labeled data and a large number of unlabeled data \citep{sohn2020fixmatch, zhang2021flexmatch, chen2023softmatch, tan2023otmatch, wang2022freematch, tan2023seal, zhang2023relationmatch}. FixMatch \cite{sohn2020fixmatch} ingeniously integrate consistency regularization with pseudo-labeling techniques. MixMatch \cite{berthelot2019mixmatch} amalgamates leading SSL methodologies, achieving a substantial reduction in error rates and bolstering privacy protection. FlexMatch \cite{zhang2021flexmatch} introduces Curriculum pseudo-labeling to improve semi-supervised learning by dynamically adapting to the model's learning status, showing notable efficacy in challenging scenarios with limited labeled data. Softmatch \cite{chen2023softmatch} efficiently balances the quantity and quality of pseudo-labels in semi-supervised learning, demonstrating significant performance improvements in diverse applications including image and text classification. FreeMatch \cite{wang2022freematch} innovates in semi-supervised learning by self-adaptively adjusting confidence thresholds and incorporating class fairness regularization, significantly outperforming existing methods in scenarios with scarce labeled data. How to more accurately utilize the information of unlabeled data remains an important problem in the field of semi-supervised learning.

\section{Preliminaries}

\subsection{Supervised classification problem}
Given a labeled dataset $\{(\mathbf{x}_i , y_i )  \}^n_{i=1}$, where $y_i \in \{1, 2, \cdots, C   \}$ is the class label. In this paper, we mainly consider training an image classification problem by concatenation of a deep neural network $h$ and a linear classifier. The linear classifier consists of a weight matrix $\mathbf{W} \in \mathbb{R}^{{C \times d}}$ and $\mathbf{b} \in \mathbb{R}^{{C \times 1}}$. Denote $\mathbf{W}^T = [w_1 \cdots w_C]$. The training loss is the cross-entropy loss.
$$
\mathcal{H}(p, q) = -\sum_{i=0}^n{p(x_i)\log q(x_i)},
$$ 
where $p$ is the true probability distribution, and $q$ is the predicted probability distribution.
 
\subsection{Matrix entropy and mutual information}

The following definitions of matrix entropy and matrix mutual information are taken from paper \citep{skean2023dime}.

\begin{definition}[Matrix entropy] Suppose a positive-definite matrix $\mathbf{K} \in \mathbb{R}^{d \times d}$ which ${\mathbf{K}(i, i)}=1$ ($1 \leq i \leq d$). The matrix entropy is defined as follows:
$$
\operatorname{H}\left(\mathbf{K}\right)=-\operatorname{tr}\left(\frac{1}{d} \mathbf{K} \log  \frac{1}{d} \mathbf{K} \right).
$$
    
\end{definition}

In the following we assume that $\mathbf{K}_j \in \mathbb{R}^{d \times d}$ which ${\mathbf{K}_j(i, i)}=1$ ($1 \leq i \leq d$, $j=1,2$).

\begin{definition}[Matrix mutual information]
The matrix mutual information is defined as follows:
$$
\operatorname{MI}\left(\mathbf{K}_1, \mathbf{K}_2\right) = \operatorname{H}\left(\mathbf{K}_1\right) + \operatorname{H}\left(\mathbf{K}_2\right) - 
 \operatorname{H}(\mathbf{K}_1 \odot \mathbf{K}_2) ,
$$
where $\odot$ is the Hardmard product.

\end{definition}

Based on the two definitions above, we can introduce the following concepts, which measure the normalized information interactions between matrices.

\begin{definition}[Matrix mutual information ratio (MIR)] \label{MIR}
The matrix mutual information ratio is defined as follows:
$$
\operatorname{MIR}\left(\mathbf{K}_{1}, \mathbf{K}_2 \right) = \frac{\operatorname{MI}\left(\mathbf{K}_1, \mathbf{K}_2\right)}{\min \{ \operatorname{H}(\mathbf{K}_1), \operatorname{H}(\mathbf{K}_2) \}}.
$$
    
\end{definition}

\begin{definition}[Matrix entropy difference ratio (HDR)] \label{HDR}
The matrix entropy difference ratio is defined as follows:
$$
\operatorname{HDR}\left(\mathbf{K}_{1}, \mathbf{K}_2 \right) = \frac{| \operatorname{H}(\mathbf{K}_1) - \operatorname{H}(\mathbf{K}_2) |}{\max \{ \operatorname{H}(\mathbf{K}_1), \operatorname{H}(\mathbf{K}_2) \}}.
$$
    
\end{definition}

\section{Theoretic Insights in Supervised Learning}

\subsection{Neural collapse}

Neural Collapse (NC) is an interesting phenomenon \citep{papyan2020prevalence} appeared at the terminal phase of the classification problem. We will briefly summarize the $3$ most important NC conditions for our paper as follows.

Denote $\mu_G =  \frac{\sum^n_{i=1}  h(\mathbf{x}_i)}{n}$ and $\mu_c = \frac{\sum_{y_i=c}  h(\mathbf{x}_i)}{\# \{ y_i=c\}}$ be the global mean and class-wise mean respectively. Then we can define $\tilde{\mu}_c =  \mu_c - \mu_G$.

(NC 1) $h(\mathbf{x}_i) =  \mu_{y_i}$  ($i=1,2,\cdots,n$). 

(NC 2) $\text{cos} (\tilde{\mu}_i , \tilde{\mu}_j) =  \frac{C}{C-1} \delta^i_j - \frac{1}{C-1}$, where $\text{cos}$ is the cosine similarity and $ \delta^i_j $ is Kronecker symbol.

(NC 3) $\frac{\mathbf{W}^T}{\| \mathbf{W}\|_F} =  \frac{\mathbf{M}}{\| \mathbf{M} \|_F}$, where $\mathbf{M}=[\tilde{\mu}_1 \cdots \tilde{\mu}_C]$.

In this paper, the matrics used in the matrix information quantities usually is the similarity (gram) matrix. For ease of exposition, we introduce a standard way of constructing a similarity (gram) matrix as follows.

\begin{definition}[Construction of similarity (gram) matrix] \label{gram}
Given a set of representations $\mathbf{Z} = [\mathbf{z}_1 \cdots \mathbf{z}_N] \in \mathbb{R}^{d \times N}$. Denote the $l_2$ normalized feature $\hat{\mathbf{z}}_i = \frac{\mathbf{z}_i}{\| \mathbf{z}_i \|}, $ $ \hat{\mathbf{Z}} = [\hat{\mathbf{z}}_1 \cdots \hat{\mathbf{z}}_N] $. Then gram matrix is defined as $\mathbf{G}(\mathbf{Z}) = \hat{\mathbf{Z}}^T\hat{\mathbf{Z}}$.
    
\end{definition}

Note that Neural Collapse conditions impose structural information on the weight matrix and class means, we provide the matrix mutual information ratio and matrix entropy difference ratio of this imposed structure in Theorem \ref{direct NC}.

\begin{theorem} \label{direct NC}
Suppose Neural collapse happens. Then $\operatorname{HDR}(\mathbf{G}(\mathbf{W}^T), \mathbf{G}(\mathbf{M})) = 0$ and $\operatorname{MIR}(\mathbf{G}(\mathbf{W}^T), \mathbf{G}(\mathbf{M})) = \frac{1}{C-1} + \frac{(C-2)\log(C-2)}{(C-1)\log(C-1)}$. 
\end{theorem}

The proof can be seen in Appendix \ref{direct NC Appendix}. As the linear weight matrix $\mathbf{W}$ can be seen as (prototype) embedding for each class. It is natural to consider the mutual information and entropy difference between sample embedding and label embedding. We discuss this in the following Corollary \ref{feature NC}.

\begin{corollary} \label{feature NC}
Suppose the dataset is class-balanced, $\mu_G =0$ and Neural collapse happens. Denote $\mathbf{Z}_1 = [h(\mathbf{x}_1) \cdots h(\mathbf{x}_n)] \in \mathbb{R}^{d \times n}$ and $\mathbf{Z}_2 = [w_{y_1} \cdots w_{y_n}] \in \mathbb{R}^{d \times n}$. Then $\operatorname{HDR}(\mathbf{Z}_1, \mathbf{Z}_2) = 0$ and $\operatorname{MIR}(\mathbf{Z}_1, \mathbf{Z}_2) = \frac{1}{C-1} + \frac{(C-2)\log(C-2)}{(C-1)\log(C-1)}$.
\end{corollary}

\textbf{Remark:} Note $ \frac{1}{C-1} + \frac{(C-2)\log(C-2)}{(C-1)\log(C-1)} \approx \frac{1}{C-1} + \frac{(C-2)\log(C-1)}{(C-1)\log(C-1)} = 1$ and MIR, HDR $\in [0,1]$. These facts make quantities obtained by Theorem \ref{direct NC} and \ref{feature NC} very interesting, as HDR reaches the minimum possible value and MIR approximately reaches the highest possible value.

\subsection{Some theoretical insights for HDR}

Mutual information is a very intuitive quantity in information theory. On the other hand, it seems weird to consider the difference of entropy, but we will show that this quantity is closely linked with comparing the approximation ability of different representations on the same target.

For ease of theoretical analysis, in this section, we consider the MSE regression loss.

The following Lemma \ref{approx} shows that the regression of two sets of representations $\mathbf{Z}_1$ and $\mathbf{Z}_2$ to the same target $\mathbf{Y}$ are closely related. And the two approximation errors are closely related to the regression error of $\mathbf{Z}_1$ to $\mathbf{Z}_2$.

\begin{lemma} \label{approx}
Suppose $\mathbf{W}^*_1, \mathbf{b}^*_1 = \argmin_{\mathbf{W}, \mathbf{b}}  \|\mathbf{Y} - (\mathbf{W} \mathbf{Z}_1 + \mathbf{b} \mathbf{1}_N )\|_F$. Then $\min_{\mathbf{W}, \mathbf{b}} \|\mathbf{Y} - (\mathbf{W} \mathbf{Z}_2 + \mathbf{b} \mathbf{1}_N )\|_F \leq \min_{\mathbf{W}, \mathbf{b}}  \|\mathbf{Y} - (\mathbf{W} \mathbf{Z}_1 + \mathbf{b} \mathbf{1}_N )\|_F + \| \mathbf{W}^*_1\|_F \min_{\mathbf{H}, \mathbf{\eta}}  \|\mathbf{Z}_1 - (\mathbf{H} \mathbf{Z}_2 + \mathbf{\eta} \mathbf{1}_N )\|_F$.   
\end{lemma}

The proof can be found in Appendix \ref{approx Appendix}. From Lemma \ref{approx}, we know that the regression error of $\mathbf{Z}_1$ to $\mathbf{Z}_2$ is crucial for understanding the differences of representations. We further bound the regression error with rank and singular values in the following Lemma \ref{rank and singuar}.

\begin{lemma} \label{rank and singuar}
Suppose $\mathbf{Z}_1 = [\mathbf{z}^{(1)}_1 \cdots \mathbf{z}^{(1)}_N] \in \mathbb{R}^{d{'} \times N}$ and $\mathbf{Z}_2 = [\mathbf{z}^{(2)}_1 \cdots \mathbf{z}^{(2)}_N] \in \mathbb{R}^{d \times N}$ and $\text{rank}(\mathbf{Z}_1) > \text{rank}(\mathbf{Z}_2)$. Denote the singular value of $\frac{\mathbf{Z}_1}{\sqrt{N}}$ as $\sigma_1 \geq \cdots \geq \sigma_{N}$. Then $\min_{\mathbf{H}, \mathbf{\eta}} \frac{1}{N} \|\mathbf{Z}_1 - (\mathbf{H} \mathbf{Z}_2 + \mathbf{\eta} \mathbf{1}_N )\|^2_F \geq \sum^{\text{rank}(\mathbf{Z}_1)}_{j= \text{rank}(\mathbf{Z}_2)+2} (\sigma_j)^2$.  
\end{lemma}

The proof can be found in Appendix \ref{rank and singuar Appendix}. The bound given by Lemma \ref{rank and singuar} is not that straightforward to understand. Assuming the features are normalized, we successfully derived the connection of regression error and ratio of ranks in Theorem \ref{rank ratio}.

\begin{theorem} \label{rank ratio}
Suppose $\| \mathbf{z}^{(1)}_j \|_2=1$, where ($1 \leq j \leq N$). Then lower bound of approximation error can be upper-bounded as follows:
$\sum^{\text{rank}(\mathbf{Z}_1)}_{j= \text{rank}(\mathbf{Z}_2)+2} (\sigma_j)^2 \leq \frac{\text{rank}(\mathbf{Z}_1)-\text{rank}(\mathbf{Z}_2)-1}{\text{rank}(\mathbf{Z}_1)} \leq 1-\frac{\text{rank}(\mathbf{Z}_2)}{\text{rank}(\mathbf{Z}_1)}$.
\end{theorem}

The proof can be found in Appendix \ref{rank ratio Appendix}. From \citep{wei2024large, zhang2023matrix}, $\exp{(\operatorname{H}(\mathbf{G}(\mathbf{Z}))}$ is a good approximate of $\text{rank}(\mathbf{Z})$. Then we can see that $\frac{\text{rank}(\mathbf{Z}_2)}{\text{rank}(\mathbf{Z}_1)} \approx \exp{(\operatorname{H}(\mathbf{G}(\mathbf{Z}_2)) - \operatorname{H}(\mathbf{G}(\mathbf{Z}_1)))}$, making the difference of entropy a good surrogate bound for approximation error.

\section{Information Interplay in Supervised Learning}

Inspired by matrix information theory and Neural Collapse theory, we focus more on the consistency between sample representations and class classification heads. We determine the relationships among samples by constructing a similarity matrix of the representations of dataset samples. According to NC1 and NC3, the similarity matrix between samples approximates the similarity matrix of the corresponding class centers, which is also the similarity matrix of the corresponding weights in the fully connected layer. Therefore, under Neural Collapse, the similarity relationship among samples is equivalent to the similarity relationship of the corresponding category weights in the fully connected layer. Our analysis, grounded in matrix information theory, primarily concentrates on the relationship between the representations of samples and the weights in the fully connected layer.
Due to constraints in computational resources, we approximate the dataset's matrix entropy using batch matrix entropy.

Our models are trained on CIFAR-10 and CIFAR-100. The default experimental configuration comprise training the models with an SGD optimizer (momentum of 0.9, weight decay of $5e^{-4}$), an initial learning rate of 0.03 with cosine annealing, a batch size of 64, and a total of $2^{20}$ training iterations. The backbone architecture is WideResNet-28-2 for CIFAR-10 and WideResNet-28-8 for CIFAR-100. 

\subsection{Information interplay during standard supervised learning process }

\begin{figure}[b]
    \centering
    \begin{subfigure}{0.49\linewidth}
        \centering
        \includegraphics[width=\linewidth, height=0.9\linewidth]{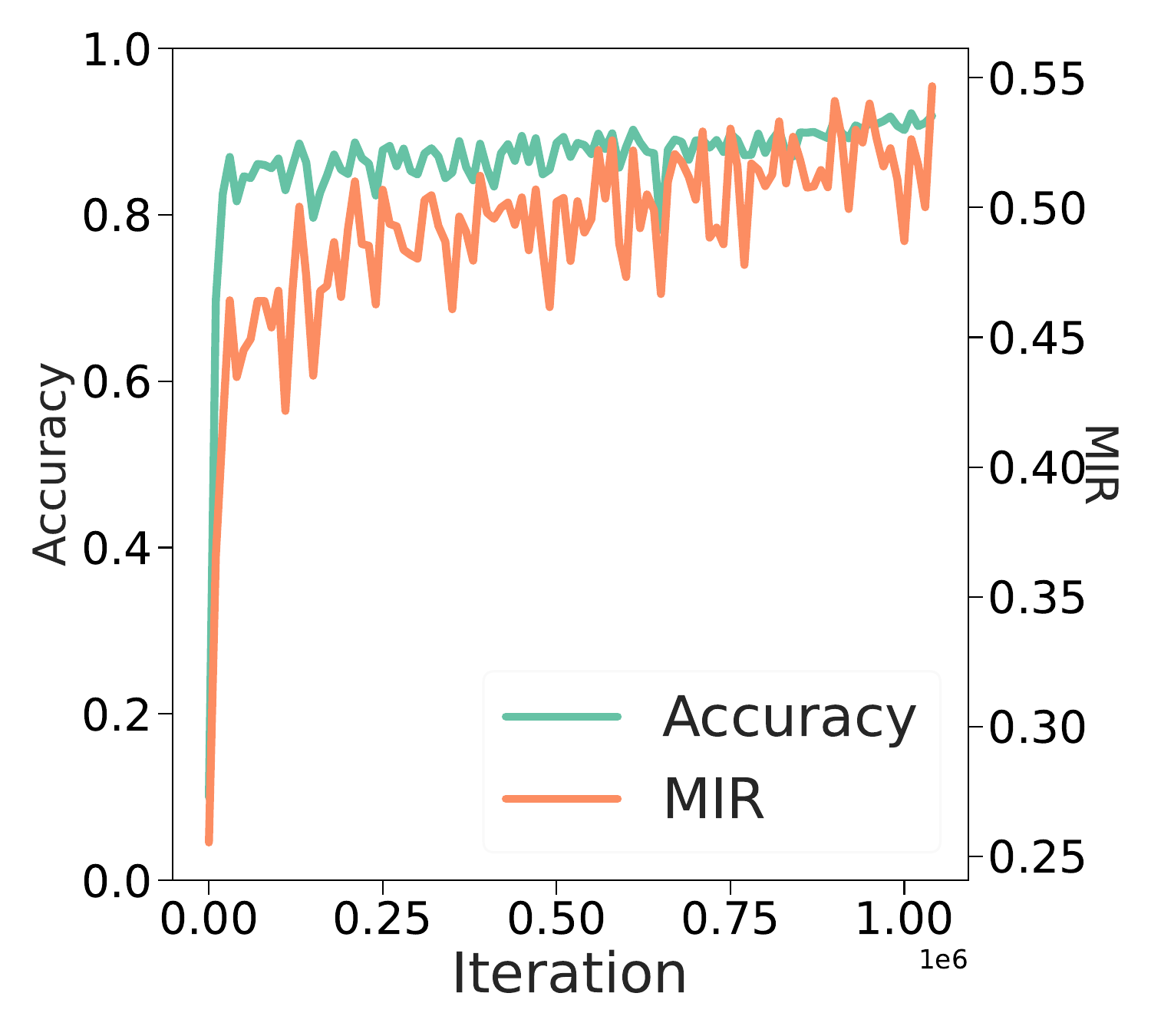}
        
        \caption{CIFAR-10}
        \label{fig:tp_MIR_10}
    \end{subfigure}
    \begin{subfigure}{0.49\linewidth}
        \centering
        \includegraphics[width=\linewidth, height=0.9\linewidth]{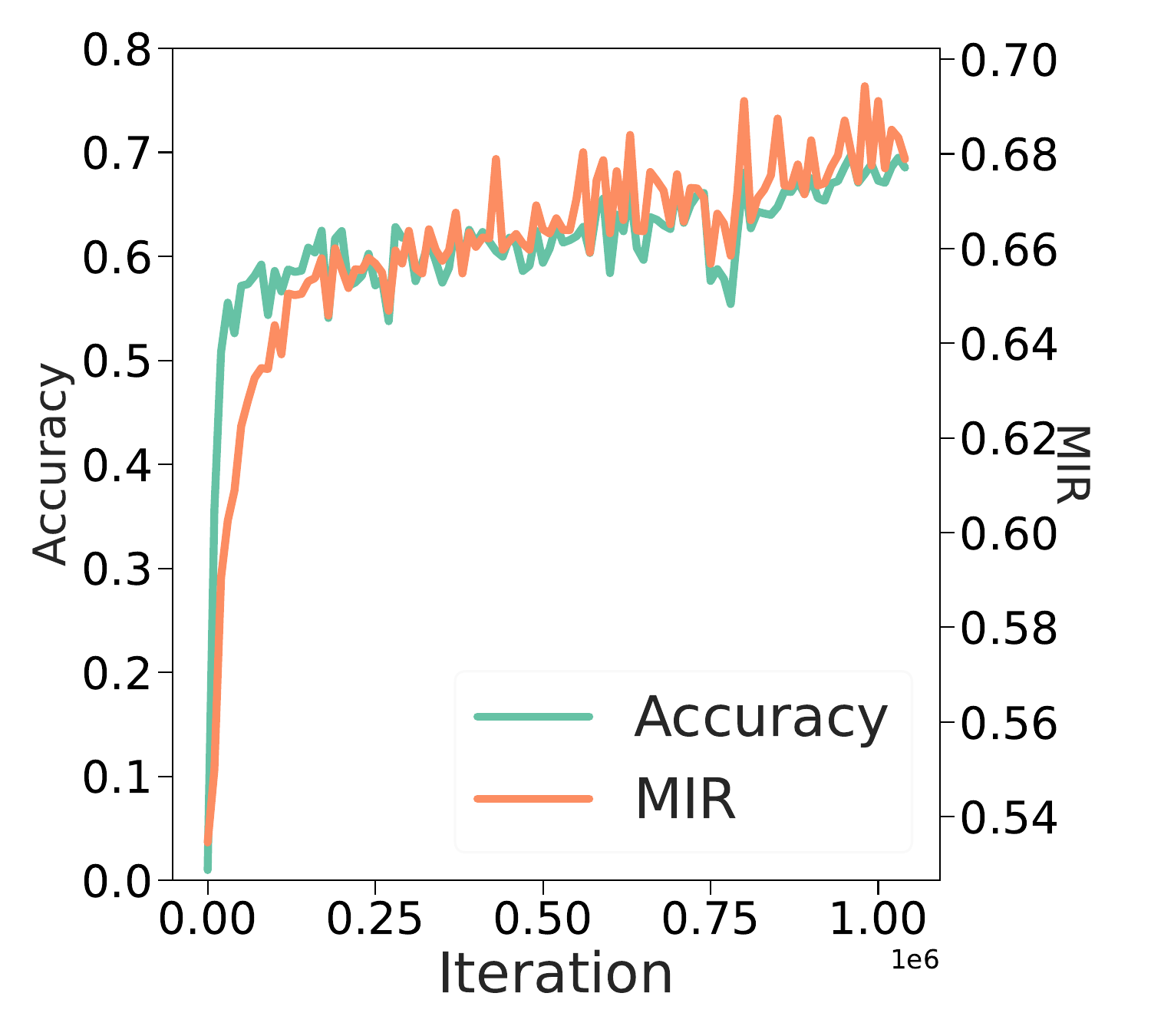}
        \caption{CIFAR-100}
        \label{fig:tp_MIR_100}
    \end{subfigure}
    \caption{Accuracy and MIR on the test set during training.}
    \label{fig:tp_MIR}
\end{figure}

According to Neural Collapse, during the terminal stages of training, sample features align with the weights of the fully connected layer. Theorem \ref{direct NC} indicates that during the training process, MIR increases to its theoretical upper limit, while HDR decreases to 0. We plot the model's accuracy on the test set during the training process, as well as the MIR and the HDR between data representations and the corresponding classification heads. As shown in Figure \ref{fig:tp_MIR}, on CIFAR-10 and CIFAR-100, the accuracy and MIR exhibit almost identical trends of variations. In most cases, both accuracy and MIR increase or decrease simultaneously, and MIR consistently shows an upward trend, having its trajectory toward its theoretical maximum value. As shown in Figure \ref{fig:tp_HDR}, during the training process, in most instances, accuracy and HDR show opposite trends, with HDR continually decreasing and even nearing its theoretical minimum value of 0 on CIFAR-100. In summary, MIR and HDR effectively describe the process of training towards Neural Collapse.

\begin{figure}[h]
    \centering
    \begin{subfigure}{0.49\linewidth}
        \centering
        \includegraphics[width=\linewidth]{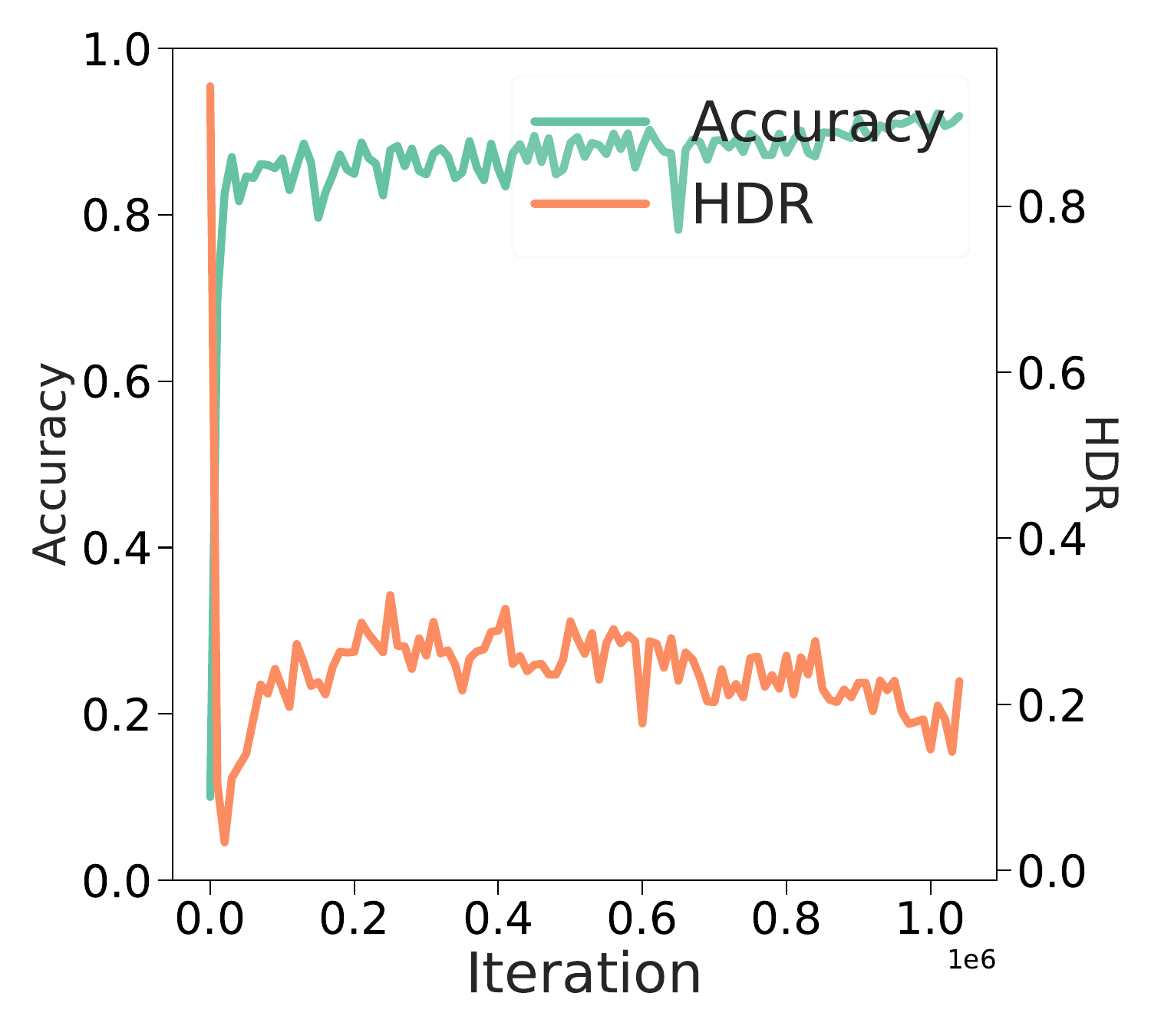}
        \caption{CIFAR-10}
        \label{fig:tp_HDR_10}
    \end{subfigure}
    \begin{subfigure}{0.49\linewidth}
        \centering
        \includegraphics[width=\linewidth]{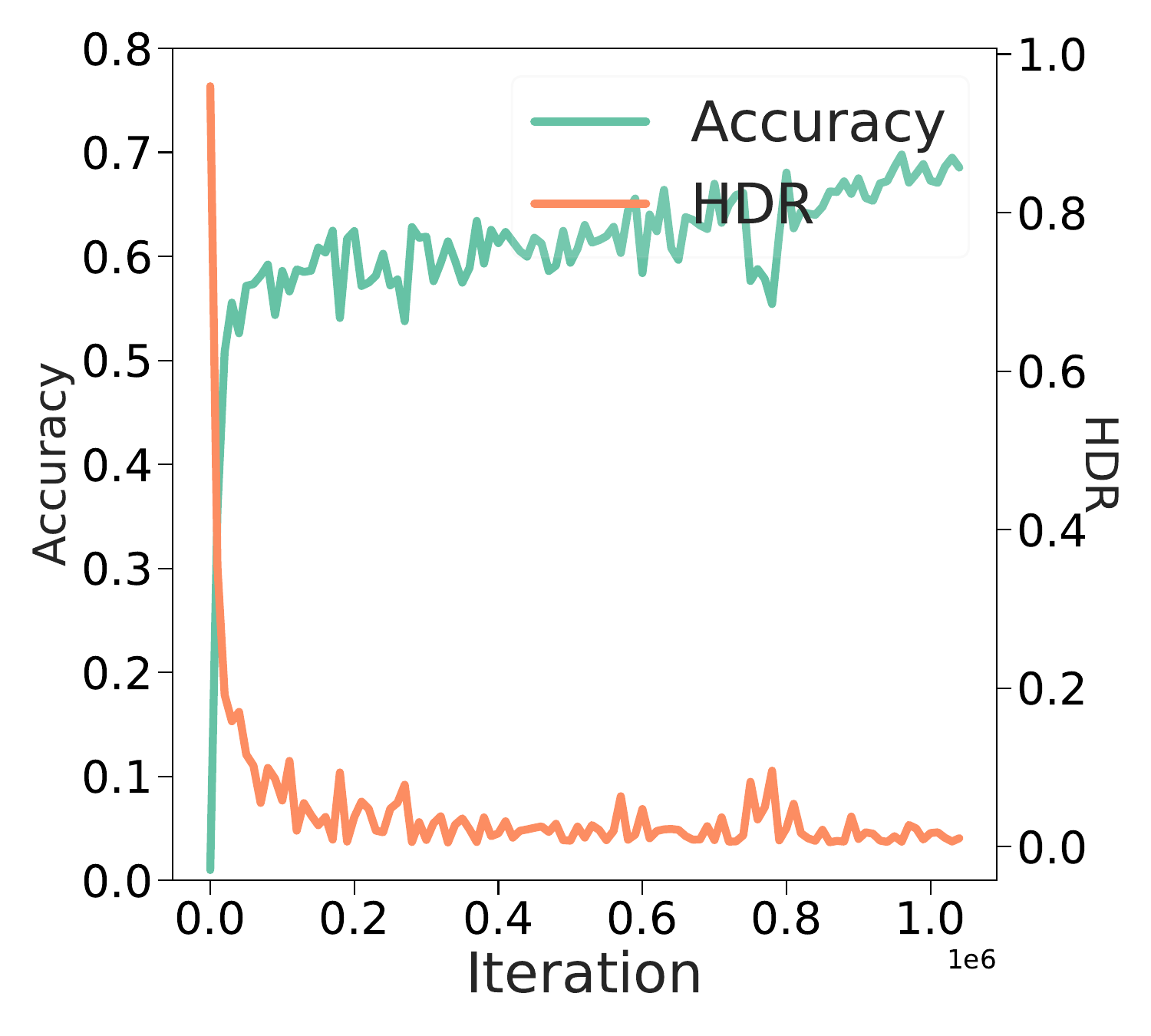}
        \caption{CIFAR-100}
        \label{fig:tp_HDR_100}
    \end{subfigure}
    \caption{Accuracy and HDR on the test set during training.}
    \label{fig:tp_HDR}
\end{figure}

\begin{figure}[b]
    \centering
    \begin{subfigure}{0.49\linewidth}
        \centering
        \includegraphics[width=\linewidth]{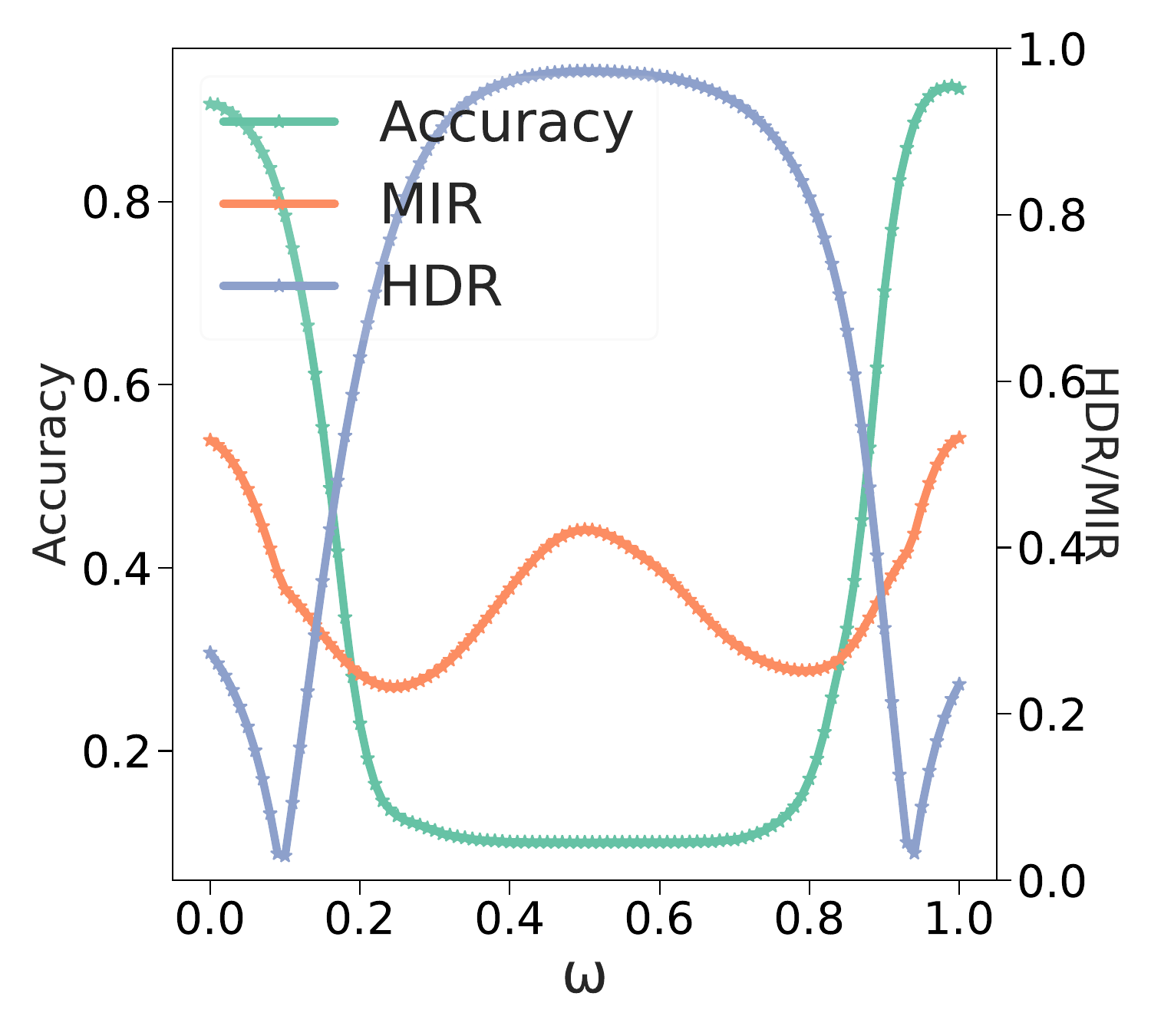}
        \caption{CIFAR-10}
        \label{fig:lc_HDRMIR_10}
    \end{subfigure}
    \begin{subfigure}{0.49\linewidth}
        \centering
        \includegraphics[width=\linewidth]{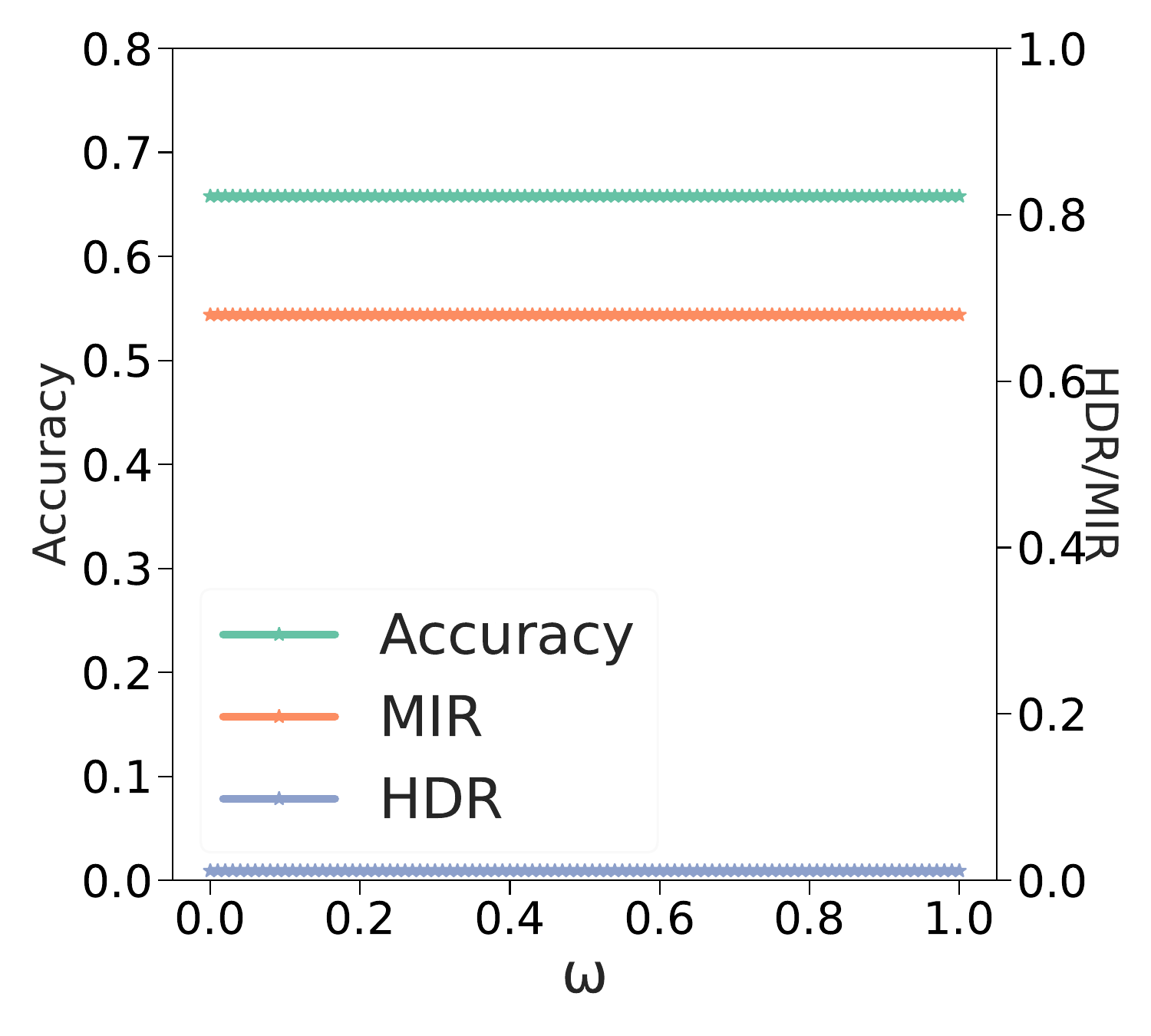}
        \caption{CIFAR-100}
        \label{fig:lc_HDRMIR_100}
    \end{subfigure}
    \caption{Accuracy, MIR, and HDR of models interpolated with different weights on the test set.}
    \label{fig:lc_HDRMIR}
\end{figure}

\begin{figure}[h]
    \centering
    \begin{subfigure}{0.325\linewidth}
        \centering
        \includegraphics[width=\linewidth]{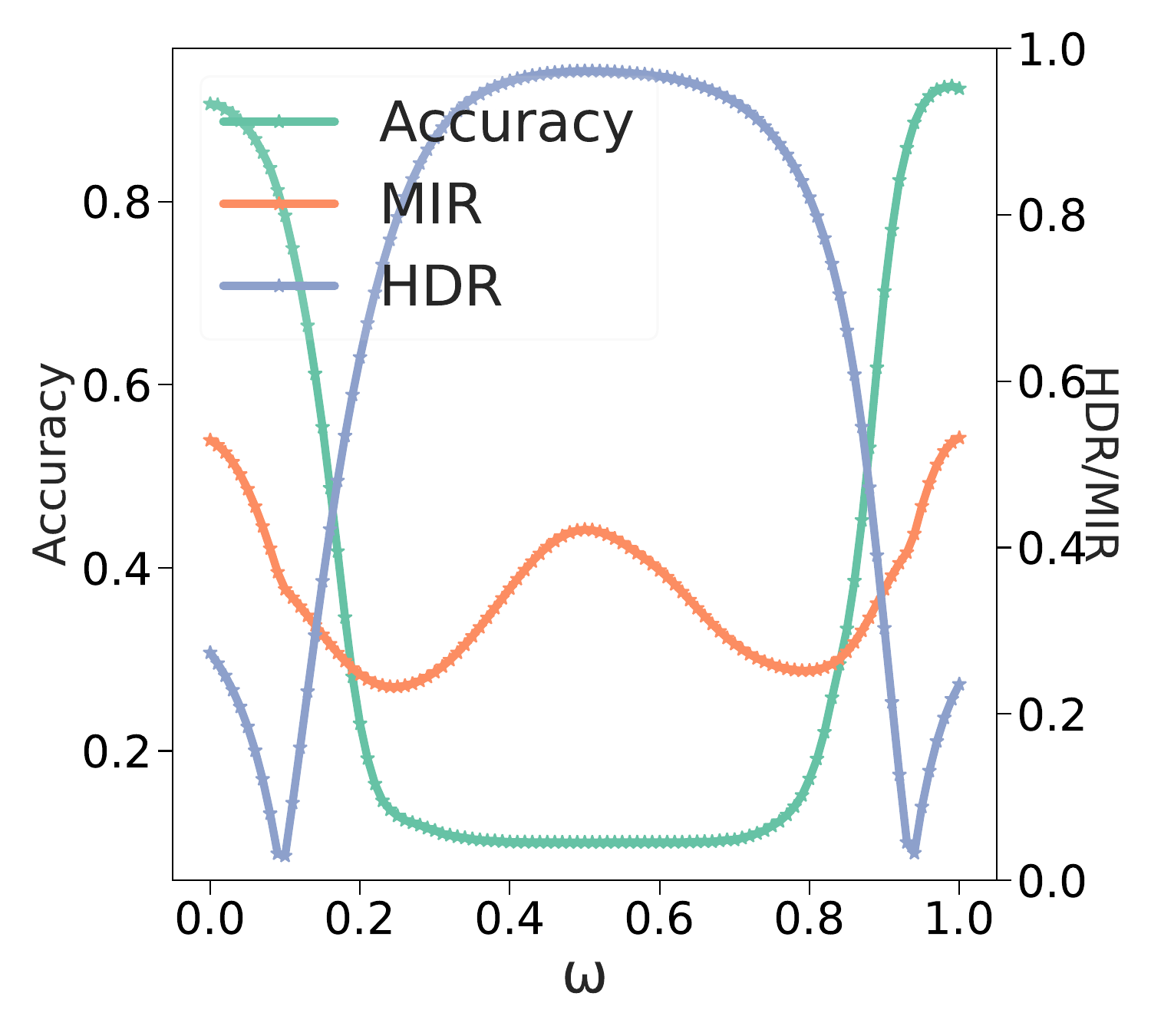}
        \caption{lr: $3e^{-2}$}
        \label{fig:lc_lr2}
    \end{subfigure}
    \begin{subfigure}{0.325\linewidth}
        \centering
        \includegraphics[width=\linewidth]{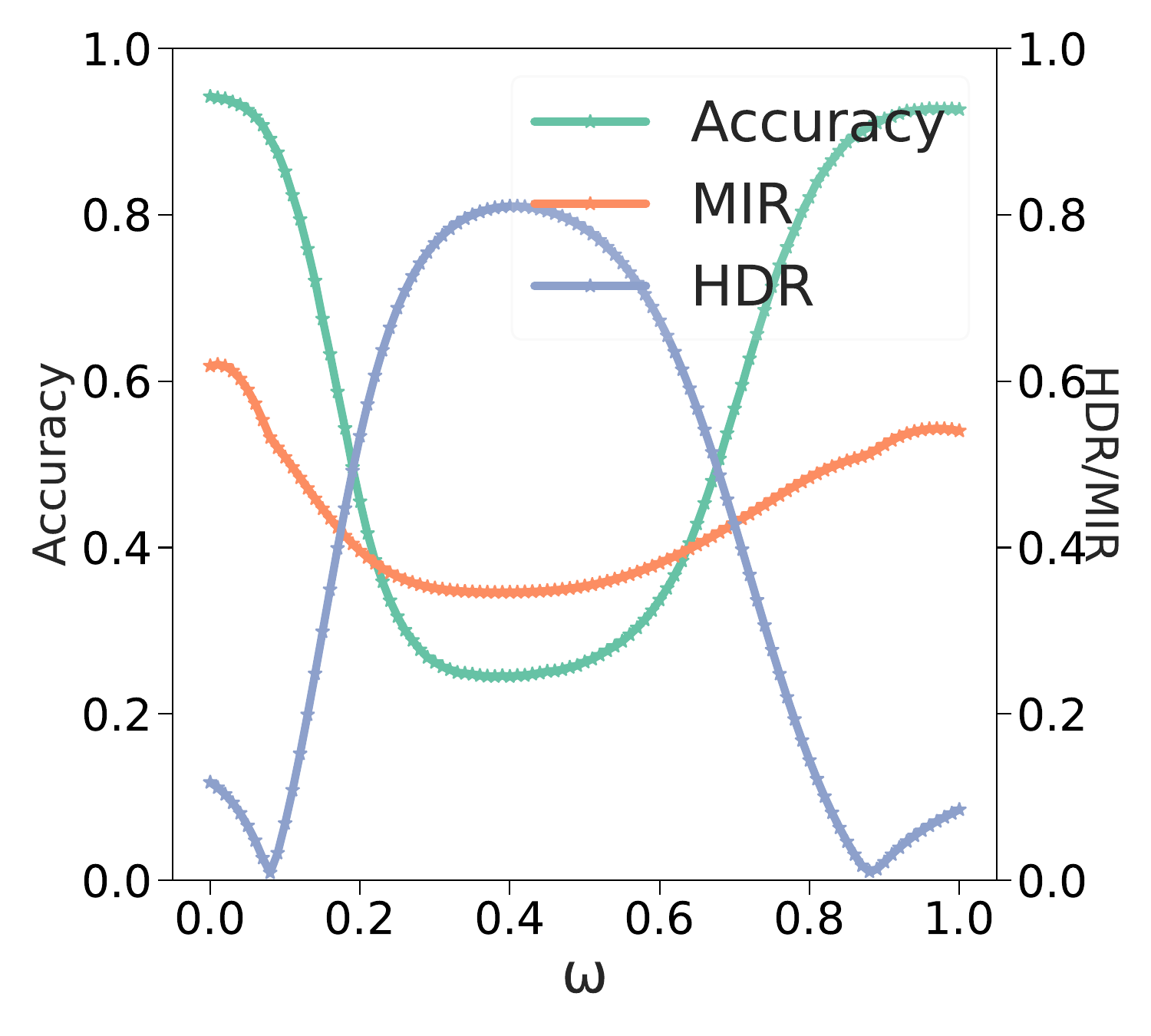}
        \caption{lr: $3e^{-3}$}
        \label{fig:lc_lr3}
    \end{subfigure}
    \begin{subfigure}{0.325\linewidth}
        \centering
        \includegraphics[width=\linewidth]{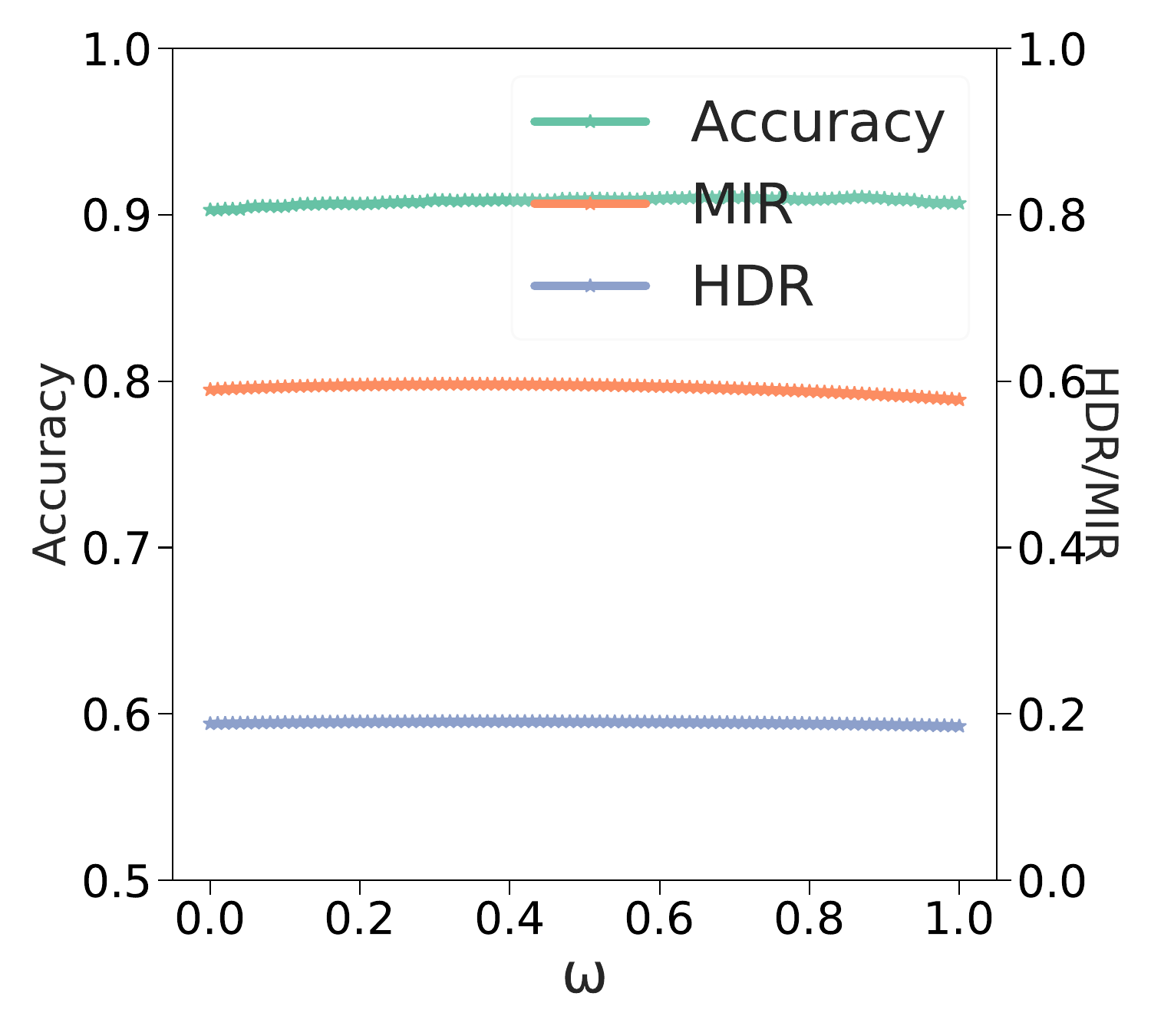}
        \caption{lr: $3e^{-4}$}
        \label{fig:lc_lr4}
    \end{subfigure}
    \caption{Accuracy, MIR, and HDR of models interpolated with different weights on CIFAR-10 test set.}
    \label{fig:lc_lr}
\end{figure}

\subsection{{Information interplay in linear mode connectivity}}
Linear mode connectivity \citep{frankle2020linear} suggests that under specific datasets and experimental setups, models initialized with the same random parameters will be optimized near the same local optimal basin, even if the order of training data and data augmentation differs. We investigate the behaviors of MIR and HDR under the setting of linear mode connectivity. We initialize models with the same random parameters and train them using different data sequences and random augmentations. Subsequently, we linearly interpolate these two checkpoints and obtain a new model $h=(1-\omega)\cdot h_1 + \omega \cdot h_2$, where $h_1$ and $h_2$ are the two ckeckpoints and $\omega$ is the interpolation weight. Then we test these models on a test set for accuracy, MIR, and HDR.

We conduct experiments on CIFAR-10 and CIFAR-100. As shown in Figure \ref{fig:lc_HDRMIR_10} and \ref{fig:lc_HDRMIR_100}. On CIFAR-100, the performance of models obtained along the interpolation line is close, aligning with the linear mode connectivity. At this point, MIR and HDR remain almost unchanged. However, on CIFAR-10, the models do not exhibit linear mode connectivity. When the value of interpolation weight is between 0.4 and 0.6, the performance of the interpolated models even drop to that of random guessing. Surprisingly, at this time, MIR shows an additional upward trend. Moreover, when the value of interpolation weight is close to 0 and 1, despite a slight decrease in performance, HDR also decreases. Although we find it difficult to explain this anomaly, it does demonstrate that HDR and MIR have distinctive attributes compared to the accuracy metric, presenting an intriguing avenue for further exploration.

\citet{altintacs2023disentangling} point that linear mode connectivity is related to the experimental configuration. Therefore, we posit that the performance decline of the interpolated model on CIFAR-10 is associated with an excessively high learning rate. In the training phase, models navigate the loss landscapes in search of minimal values, and two models with linear mode connectivity are optimized near the same local optimum. When the learning rate is too high, different training sample ordering and data augmentations lead to directing model optimization towards distinct regions within the loss landscapes. We experiment with different learning rates on CIFAR-10 and test their linear mode connectivity. It is observed that as the learning rate decreased, fluctuations in accuracy, MIR, and HDR also reduced. When the learning rate is lowered to $3e^{-4}$, the model demonstrate linear mode connectivity on CIFAR-10. This suggests that HDR and MIR are also effective in describing linear mode connectivity when it exists.

\subsection{Information interplay in label smoothing}

\begin{figure}[b]
    \centering
    \begin{subfigure}{0.49\linewidth}
        \centering
        \includegraphics[width=\linewidth]{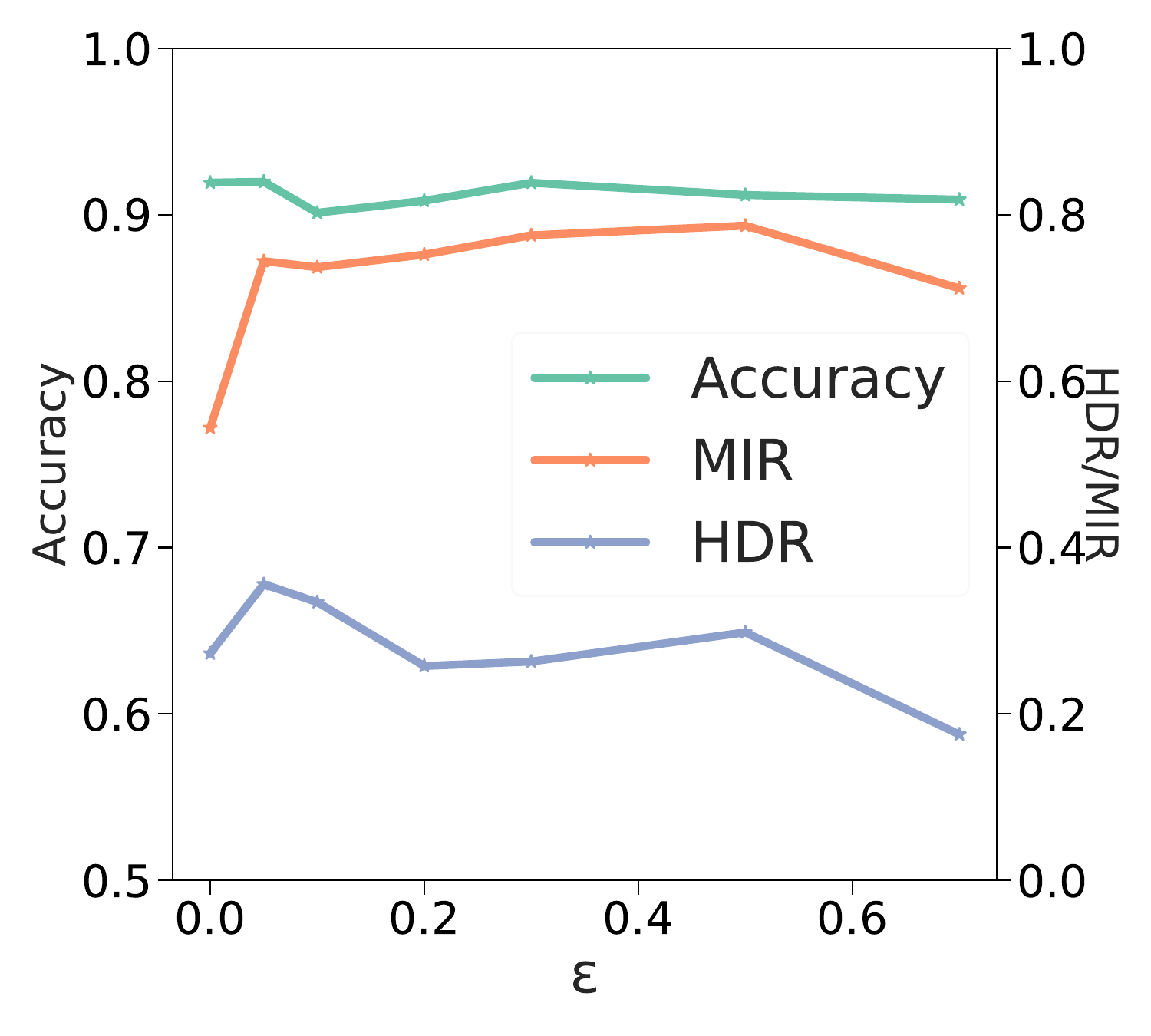}
        \caption{CIFAR-10}
        \label{fig:ls_HDRMIR_10}
    \end{subfigure}
    \begin{subfigure}{0.49\linewidth}
        \centering
        \includegraphics[width=\linewidth]{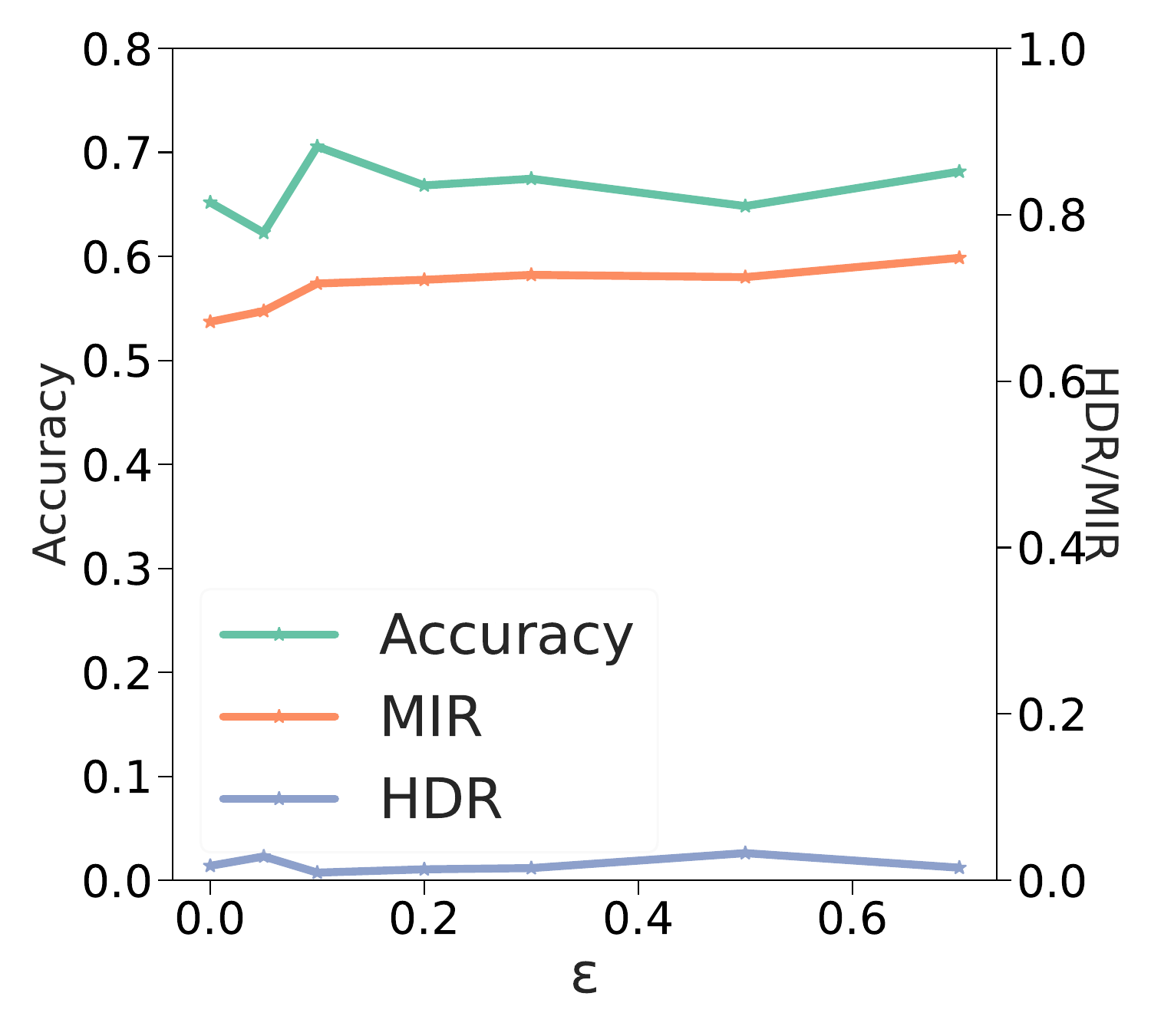}
        \caption{CIFAR-100}
        \label{fig:ls_HDRMIR_100}
    \end{subfigure}
    \caption{Accuracy, MIR, and HDR under different smoothness levels.}
    \label{fig:ls_HDRMIR}
\end{figure}

Label smoothing \cite{szegedy2016rethinking} is a widely used technique in deep learning. It improves the generalization of the model by setting smoothness of labels. $y'=(1-\epsilon)\cdot y + \frac{\epsilon}{C}$, where $\epsilon$ is the smoothness and $y$ is the one-hot label. We train models with various smoothness levels to explore their impact on accuracy, HDR, and MIR. As shown in Figure \ref{fig:ls_HDRMIR}, the variation in accuracy, MIR, and HDR are minimal, indicating that HDR and MIR can effectively describe the performance of label smoothing technique.

\begin{figure}[h]
    \centering
    \begin{subfigure}{0.49\linewidth}
        \centering
        \includegraphics[width=\linewidth]{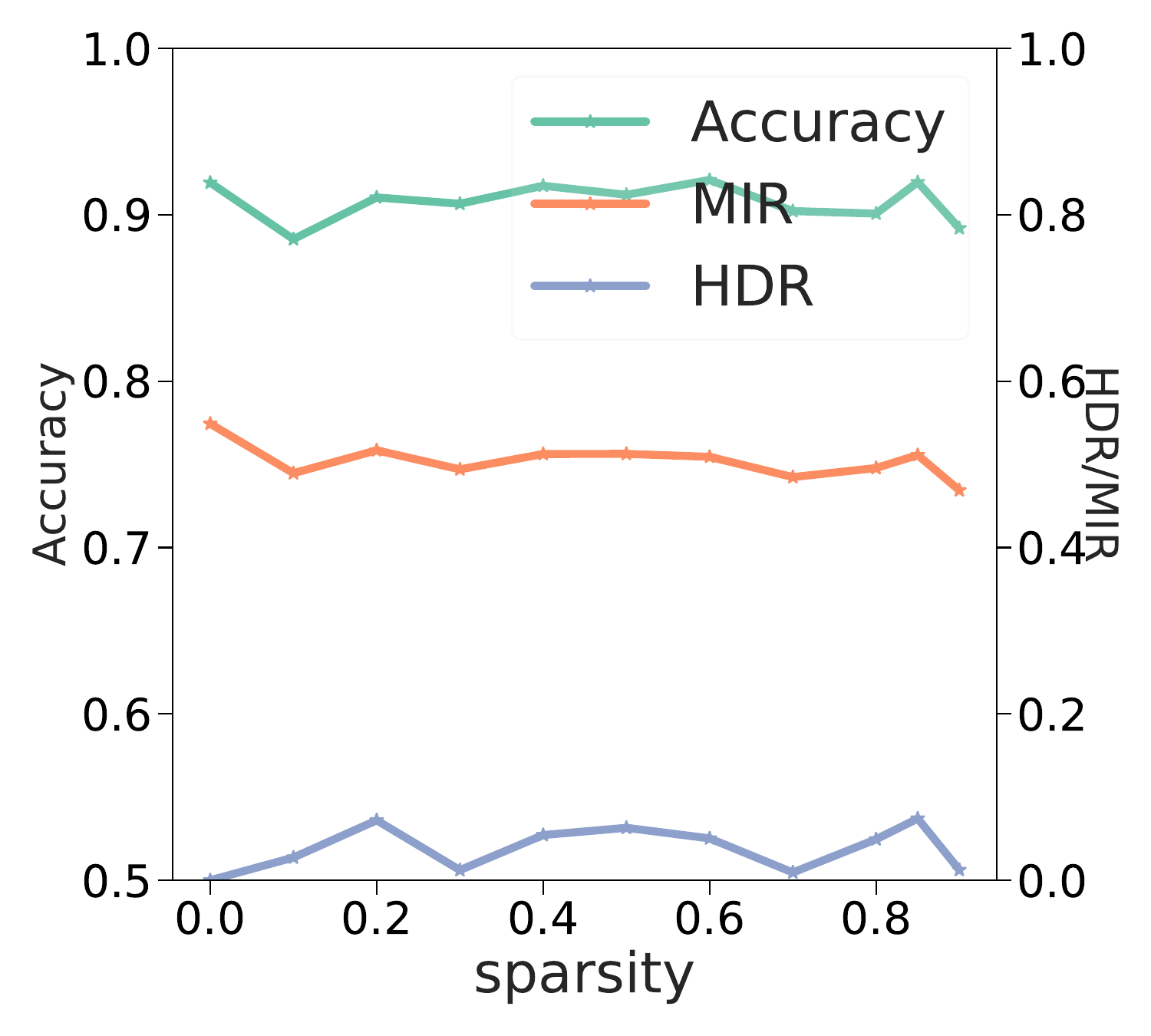}
        \caption{CIFAR-10}
        \label{fig:prune_HDRMIR_10}
    \end{subfigure}
    \begin{subfigure}{0.49\linewidth}
        \centering
        \includegraphics[width=\linewidth]{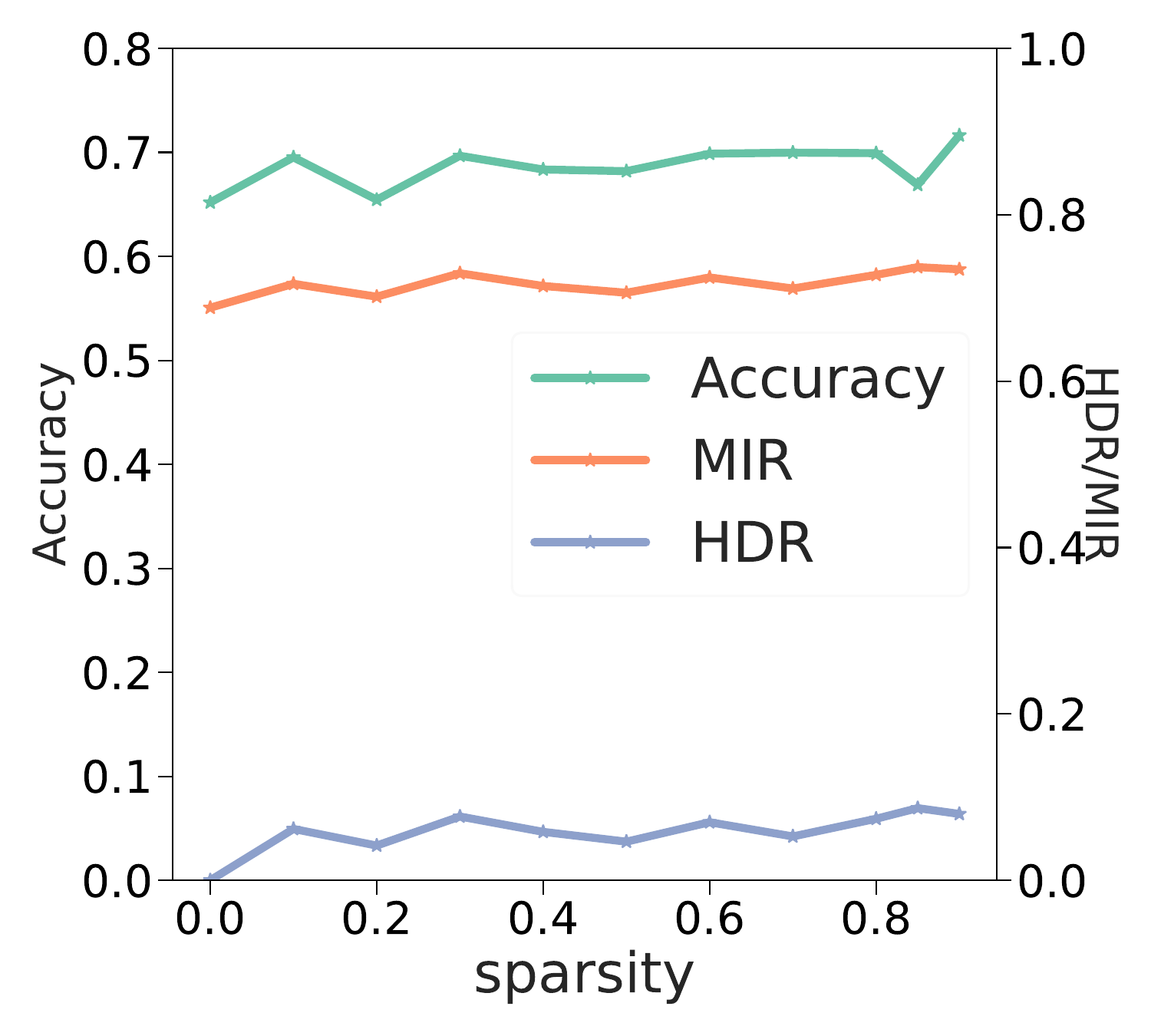}
        \caption{CIFAR-100}
        \label{fig:prune_HDRMIR_100}
    \end{subfigure}
    \caption{Accuracy, MIR, and HDR of the features extracted by the model before and after pruning.}
    \label{fig:prune_HDRMIR}
    \vspace{-20pt}
\end{figure}

\subsection{{Information interplay in model pruning}}

We would like to use MIR and HDR to understand why the pruning technique is effective in maintaining relatively high accuracy. We apply standard unstructured pruning to models: for a well-trained model, we determine the number of parameters to prune, denoted as $k$. The model's parameters are sorted in descending order based on their absolute values, and the smallest $k$ parameters are removed. Subsequently, the remaining parameters are fine-tuned again on the dataset \cite{han2015learning}.

\begin{figure}[b]
    \centering
    \includegraphics[width=0.49\linewidth, trim=0 0 0 0, clip]{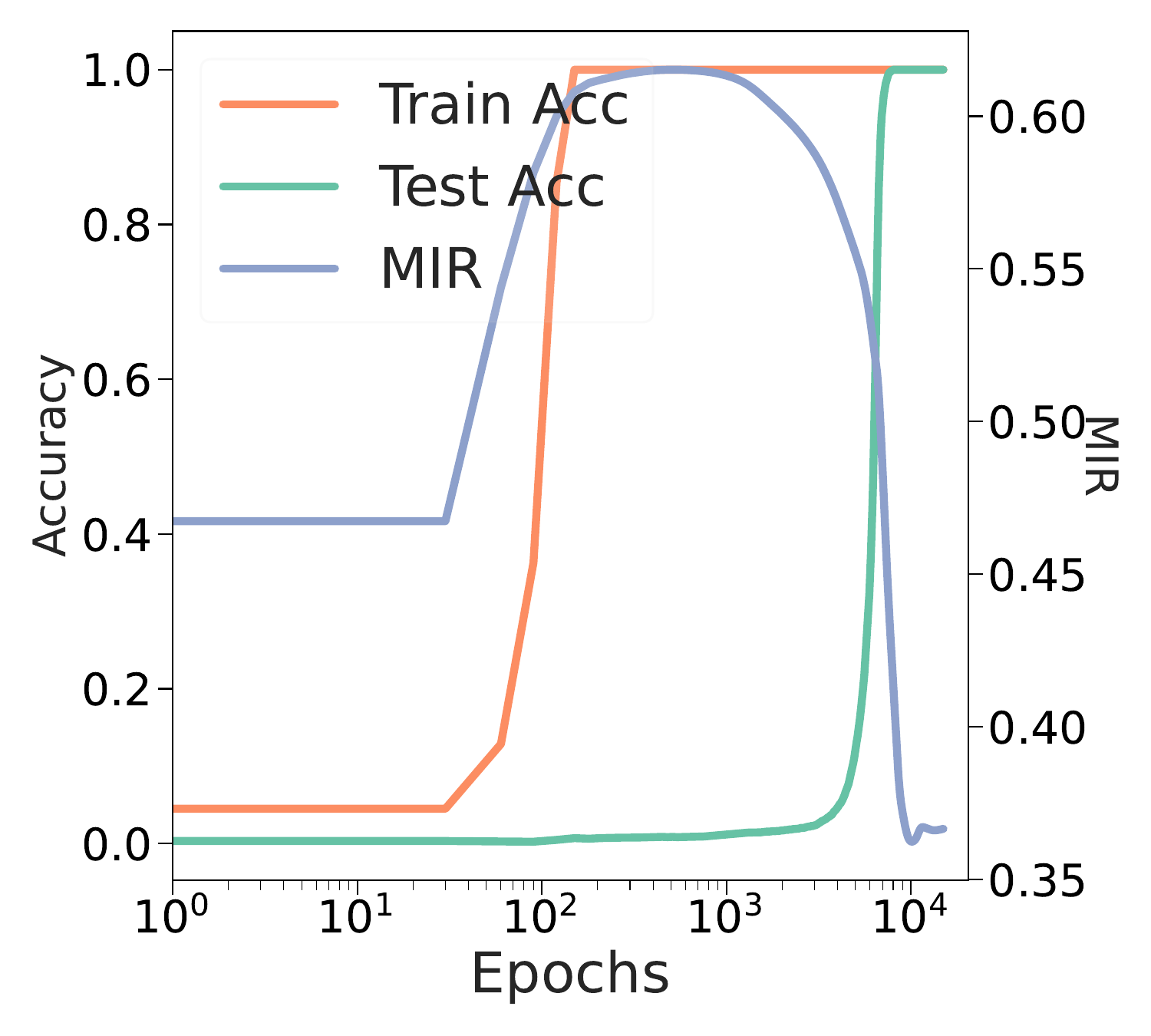}
    \includegraphics[width=0.49\linewidth, trim=0 0 0 0, clip]{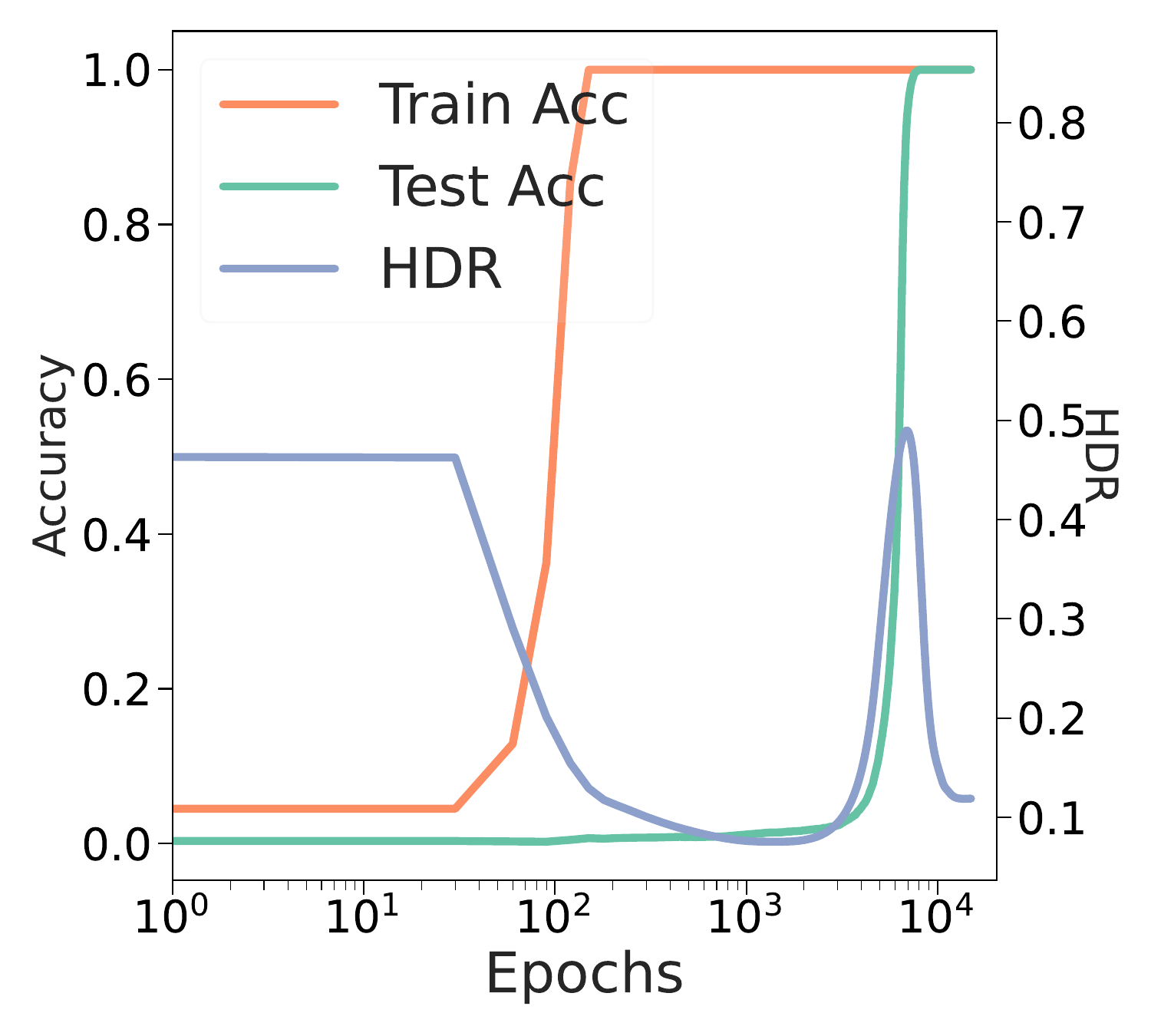}
    \caption{Accuracy, MIR, and HDR during Grokking.}
    \label{fig:grokking}
    
\end{figure}

We prune the model under various sparsity levels and extract features using the model before and after pruning. We calculate the MIR and HDR of the features extracted by the model before and after pruning. As shown in Figure \ref{fig:prune_HDRMIR}, even when the model sparsity is 90\%, the features extracted by the pruned model maintain a high MIR with those extracted before pruning.  At various sparsity ratio, the variations in MIR and HDR for models after pruning are little compared to these before pruning, the performance differences of the models before and after pruning are also not significant. This indicates that fine-tuning the pruned subnetwork can restore adequate information extraction capabilities.

\subsection{{Information interplay in grokking}}

In supervised learning, training models on certain datasets can result in an anomalous situation. Initially, models quickly learn the patterns of the training set, but at this point, their performance on the test set is very poor. As training continues, the models learn representations that can generalize to the test set, a phenomenon referred to as Grokking\cite{nanda2022progress}. We aim to explore the information interplay in Grokking. Following \citet{nanda2022progress, tan2023understanding}, we train a transformer to learn modular addition $ c \equiv (a + b)\pmod{p} $, with $p$ being 113. The model input is ``$ a~b = $'', where $a$ and $b$ are encoded into $p$-dimensional one-hot vectors, and ``$=$'' is used to signify the output value $c$. Our model employ a single-layer ReLU transformer with a token encoding dimension of 128 to learn positional encodings, four attention heads each of dimension 32, and an MLP with a hidden layer of dimension 512. We train the model using full-batch gradient descent, a learning rate of 0.001, and an AdamW optimizer with a weight decay parameter of 1. We use 30\% of all possible inputs ($113 \times 113$ pairs) for training data and test performance on the remaining 70\%.

As shown in Figure \ref{fig:grokking}, we plot the accuracy of both the training and test sets during the grokking process, as well as the variation in MIR and HDR between the representation and the fully connected layer. It can be observed that, in the early stages of training, the model quickly fits the training data and achieves 100\% accuracy on the training set. However, at this point, the performance on the test set is nearly equivalent to that of random guessing. As training continues, the model gradually shows generalization capability on the test set, ultimately achieving 100\% accuracy, which is a hallmark of grokking. Figure \ref{fig:grokking} also reveals a clear two-phase variation in both MIR and HDR between data representation and the weight of fully connected layer. Initially, similar to fully supervised learning, MIR increases, while HDR decreases. However, as training proceeds, MIR begins to decrease, and HDR starts to increase, indicating the model is seeking new optimal points. After the model achieves the grokking, MIR reaches its lowest, and HDR rapidly declines from its highest point. The experiments demonstrate that HDR and MIR exhibit distinct phenomena in two stages, suggesting that information metrics can describe the grokking phenomenon, providing a basis for further research.

\section{Improving Supervised and Semi-Supervised Learning With Information Interplay}

\subsection{Pipeline of supervised and semi-supervised learning}
In this section, we introduce how to apply matrix information entropy in supervised and semi-supervised learning. In supervised learning, we train the nerual network $h$ and classifier $\mathbf{W} \in \mathbb{R}^{C\times d}$ on the dataset $\mathcal{D}_L=\{(x_i,~y_i)\}_{i=0}^{N_L}$ consisting of $N_L$ samples. $h$ is used to extract data features $f \in \mathbb{R}^D$, and $\mathbf{W}$ classifies the extracted features. The model is optimized using the following cross-entropy loss.

\begin{equation*}
    \mathcal{L}_s=\frac{1}{B}\sum_{i=1}^B\mathcal{H}(y_i, p(\omega(x_i))),
\end{equation*}
where $B$ represents the batch size, $\mathcal{H}$ denotes the cross-entropy loss, $p(\cdot)$ refers to the model's output probability of a sample, and $\omega$ means random data augmentation.

Compared to supervised learning, semi-supervised learning includes an additional unlabeled dataset $\mathcal{D}_U=\{u_i\}_{i=0}^{N_U}$ which contain $N_U$ unlabeled data and utilizes it to assist in optimizing the model. In the processing of unlabeled data, we adopt the approach outlined in Freematch \cite{wang2022freematch}. This involves generating pseudo-labels through weak data augmentation and selecting data based on a probability threshold. The model is then employed to extract features from strongly augmented data for the computation of cross-entropy loss in conjunction with the pseudo-labels. The formulaic representation of the training objective for unlabeled data is as follows:

\begin{equation*}
    \mathcal{L}_u=\frac{1}{\mu B}\sum_{i=1}^{\mu B}\mathbb{I}\left(max(q_i) > \tau\right)\cdot \mathcal{H}\left(\hat{q_i},Q_i\right),
\end{equation*}

where $q_i$ and $Q_i$ correspond to $p(y|\omega(u_i))$ and $p(y|\Omega(u_i))$, respectively. The term $\hat{q_i}$ refers to one-hot pseudo-labels generated from $q_i$. The symbol $\mathbb{I}(\cdot > \tau)$ denotes the indicator function applied to values surpassing the threshold $\tau$. Furthermore, $\omega$ and $\Omega$ are used to distinguish between weak and strong data augmentation.

In addition, Freematch incorporates a fairness objective to predict each class with uniform frequency. 

\begin{equation*}
    \mathcal{L}_f=-H\left(\text{SumNorm}\left(\frac{p_1}{hist_1}\right), \text{SumNorm}\left(\frac{p_2}{hist_2}\right)\right).
\end{equation*}
$\text{SumNorm}=(\cdot)/\sum(\cdot) $. $p_1$ and $p_2$ refer to the average predictions of the model under weak and strong augmentation, respectively. Likewise, $hist_1$ and $hist_2$ indicate the histogram distributions resulting from weak and strong augmentation, respectively.

The overall objective is 
\begin{equation*}
    \mathcal{L}_{ssl}=\mathcal{L}_s+\lambda_u\mathcal{L}_u+\lambda_f\mathcal{L}_f,
\end{equation*}
where $\lambda_u$ and $\lambda_f$ represent the weight for $\mathcal{L}_u$ and $\mathcal{L}_f$.

\subsection{Insights from information interplay}

In a batch of labeled data $\{(x_i, y_i)\}_{i=1}^{B} \in \mathcal{D}_L$, $h$ extracts feature representations, denoted as $f \in \mathbb{R}^{B\times D}$. In Neural Collapse theory, the representation of each sample's class center aligns with the classifier weight of the respective category, i.e., $V_{i}=W_{y_i}$. In the case of unlabeled data $\{u_i\}_{i=1}^{\mu B} \in \mathcal{D}_U$, we select sample features $f'$ from $\mu B$ samples with pseudo-label probabilities greater than $\tau$. i.e., $f'=\{f_i \in f | \mathbb{I}\left(max(q_j) > \tau\right)\},$ and obtain the corresponding class centers $V'=W_{y_i'}$, where $y_i'$ is the pseudo label of $f'$.

\textbf{Maximizing mutual information.}
As shown in Figure \ref{fig:tp_MIR}, during the model training process, the mutual information between a batch's data features $f$ and the corresponding class weights $V$ increases. Therefore, we add an additional loss term to increase the mutual information between them. For supervised learning, the final optimization objective is 
\begin{equation*}
    \mathcal{L} = \mathcal{L}_{s} - \lambda_{mi}\cdot\text{MI}\left(\mathbf{G}(f),\mathbf{G}(V)\right).
\end{equation*}
For semi-supervised learning, the final optimization objective is 
\begin{equation*}
    \mathcal{L} = \mathcal{L}_{ssl} - \lambda_{mi}\cdot\text{MI}\left(\mathbf{G}(f'),\mathbf{G}(V')\right),
\end{equation*}

where $\lambda_{mi}$ is the weight for the mutual information.

\textbf{Minimizing entropy difference.} As depicted in Figure \ref{fig:tp_HDR}, throughout the training phase, the disparity in information entropy between a batch's data features $f$ and the associated category weights $V$ diminishes in tandem with an increase in accuracy. Consequently, it is feasible to introduce an auxiliary loss component within the training regime to further mitigate this entropy discrepancy. In the context of supervised learning, the ultimate optimization target is delineated as 
\begin{equation*}
    \mathcal{L} = \mathcal{L}_{s} + \lambda_{id}\cdot\left|\text{H}(\mathbf{G}(f))-\text{H}(\mathbf{G}(V))\right|.
\end{equation*}
Regarding semi-supervised learning, this target shifts to 
\begin{equation*}
    \mathcal{L} = \mathcal{L}_{ssl} + \lambda_{id}\cdot\left|\text{H}(\mathbf{G}(f'))-\text{H}(\mathbf{G}(V'))\right|,
\end{equation*}
wherein $\lambda_{id}$ signifies the weight for entropy difference.

\subsection{Performances on supervised and semi-supervised learning}

\begin{table*}[t]
\centering
\caption{Error rates (100\% - accuracy) on CIFAR-10/100, and STL-10 datasets for state-of-the-art methods in semi-supervised learning. Bold indicates the best performance, and underline indicates the second best.}
\resizebox{\textwidth}{!}{%
\begin{tabular}{l|ccc|cc|ccc}
\toprule
Dataset & \multicolumn{3}{c|}{CIFAR-10} & \multicolumn{2}{c|}{CIFAR-100}& \multicolumn{2}{c}{STL-10} \\ 
\cmidrule{1-1}\cmidrule(lr){2-4}\cmidrule(lr){5-6}\cmidrule{7-8} 
\# Label & 10 & 40 & 250 & 400   &2500& 40 & 1000\\ 
\cmidrule{1-1}\cmidrule(lr){2-4}\cmidrule(lr){5-6}\cmidrule{7-8}
$\Pi$ Model \cite{rasmus2015semi} & 
79.18{\scriptsize $\pm$1.11} &
74.34{\scriptsize $\pm$1.76} & 46.24{\scriptsize $\pm$1.29} & 86.96{\scriptsize $\pm$0.80}  & 58.80{\scriptsize $\pm$0.66} & 74.31{\scriptsize $\pm$0.85} & 32.78{\scriptsize $\pm$0.40} \\
Pseudo Label \cite{lee2013pseudo} & 80.21{\scriptsize $\pm$ 0.55} & 74.61{\scriptsize $\pm$0.26} & 46.49{\scriptsize $\pm$2.20} & 87.45{\scriptsize $\pm$0.85}  & 57.74{\scriptsize $\pm$0.28} & 74.68{\scriptsize $\pm$0.99} & 32.64{\scriptsize $\pm$0.71} \\
VAT \cite{miyato2018virtual} & 79.81{\scriptsize $\pm$ 1.17} & 74.66{\scriptsize $\pm$2.12} & 41.03{\scriptsize $\pm$1.79} & 85.20{\scriptsize $\pm$1.40}  & 48.84{\scriptsize $\pm$0.79} & 74.74{\scriptsize $\pm$0.38} & 37.95{\scriptsize $\pm$1.12} \\
MeanTeacher \cite{tarvainen2017mean} & 76.37{\scriptsize $\pm$ 0.44} & 70.09{\scriptsize $\pm$1.60} & 37.46{\scriptsize $\pm$3.30} & 81.11{\scriptsize $\pm$1.44}  & 45.17{\scriptsize $\pm$1.06} & 71.72{\scriptsize $\pm$1.45} & 33.90{\scriptsize $\pm$1.37} \\
MixMatch \cite{berthelot2019mixmatch} & 65.76{\scriptsize $\pm$ 7.06} & 36.19{\scriptsize $\pm$6.48} & 13.63{\scriptsize $\pm$0.59} & 67.59{\scriptsize $\pm$0.66}  & 39.76{\scriptsize $\pm$0.48} & 54.93{\scriptsize $\pm$0.96} & 21.70{\scriptsize $\pm$0.68} \\
ReMixMatch  \cite{berthelot2019remixmatch} & 20.77{\scriptsize $\pm$ 7.48} & 9.88{\scriptsize $\pm$1.03} & 6.30{\scriptsize $\pm$0.05} & 42.75{\scriptsize $\pm$1.05}  & 26.03{\scriptsize $\pm$0.35} & 32.12{\scriptsize $\pm$6.24} & 6.74{\scriptsize $\pm$0.17}\\
UDA \cite{xie2020unsupervised} & 34.53{\scriptsize $\pm$ 10.69} & 10.62{\scriptsize $\pm$3.75} & 5.16{\scriptsize $\pm$0.06} & 46.39{\scriptsize $\pm$1.59}  & 27.73{\scriptsize $\pm$0.21} & 37.42{\scriptsize $\pm$8.44} & 6.64{\scriptsize $\pm$0.17} \\
FixMatch \cite{sohn2020fixmatch} & 24.79{\scriptsize $\pm$ 7.65} & 7.47{\scriptsize $\pm$0.28} & 5.07{\scriptsize $\pm$0.05} & 46.42{\scriptsize $\pm$0.82}  & 28.03{\scriptsize $\pm$0.16} & 35.97{\scriptsize $\pm$4.14} & 6.25{\scriptsize $\pm$0.33} \\
Dash \cite{xu2021dash} & 27.28{\scriptsize $\pm$ 14.09} & 8.93{\scriptsize $\pm$3.11} & 5.16{\scriptsize $\pm$0.23} & 44.82{\scriptsize $\pm$0.96}  & 27.15{\scriptsize $\pm$0.22} & 34.52{\scriptsize $\pm$4.30} & 6.39{\scriptsize $\pm$0.56} \\
MPL \cite{pham2021meta} & 23.55{\scriptsize $\pm$ 6.01} & 6.93{\scriptsize $\pm$0.17} & 5.76{\scriptsize $\pm$0.24} & 46.26{\scriptsize $\pm$1.84}  & 27.71{\scriptsize $\pm$0.19} & 35.76{\scriptsize $\pm$4.83} & 6.66{\scriptsize $\pm$0.00} \\

FlexMatch \cite{zhang2021flexmatch} & 13.85{\scriptsize $\pm$ 12.04} & 4.97{\scriptsize $\pm$0.06} & 4.98{\scriptsize $\pm$0.09} & 39.94{\scriptsize $\pm$1.62}  & 26.49{\scriptsize $\pm$0.20} & 29.15{\scriptsize $\pm$4.16} & 5.77{\scriptsize $\pm$0.18} \\
FreeMatch \cite{wang2022freematch} & {8.07{\scriptsize $\pm$ 4.24}} & {4.90{\scriptsize $\pm$0.04}} & {4.88{\scriptsize $\pm$0.18}} & {37.98{\scriptsize $\pm$0.42}}  & 26.47{\scriptsize $\pm$0.20} & {15.56{\scriptsize $\pm$0.55}} & {5.63{\scriptsize $\pm$0.15}} \\
OTMatch \cite{tan2023otmatch} & {4.89{\scriptsize $\pm$ 0.76}} & {4.72{\scriptsize $\pm$0.08}} & {4.60{\scriptsize $\pm$0.15}} & {37.29{\scriptsize $\pm$0.76}}  & 26.04{\scriptsize $\pm$0.21} & \textbf{12.10{\scriptsize $\pm$0.72}} & {5.60{\scriptsize $\pm$0.14}} \\

SoftMatch \cite{chen2023softmatch} & {4.91{\scriptsize $\pm$ 0.12}} & {4.82{\scriptsize $\pm$0.09}} & \textbf{4.04{\scriptsize $\pm$0.02}} & {37.10{\scriptsize $\pm$0.07}}  & 26.66{\scriptsize $\pm$0.25} & {21.42{\scriptsize $\pm$3.48}} & {5.73{\scriptsize $\pm$0.24}} \\

\cmidrule{1-1}\cmidrule(lr){2-4}\cmidrule(lr){5-6}\cmidrule{7-8}

FreeMatch + Maximizing Mutual Information (Ours) & \underline{4.87{\scriptsize $\pm$ 0.66}} & \underline{4.66{\scriptsize $\pm$ 0.13}} & \underline{4.56{\scriptsize $\pm$ 0.15}} & \textbf{36.41{\scriptsize $\pm$ 1.91}}   & \textbf{25.77{\scriptsize $\pm$ 0.35}} & {16.61{\scriptsize $\pm$ 1.19}} & \textbf{{5.24 \scriptsize $\pm$ 0.17}} \\

FreeMatch + Minimizing Entropy Difference (Ours) & \textbf{4.69{\scriptsize $\pm$ 0.16}} & \textbf{4.63{\scriptsize $\pm$ 0.25}} & {4.60{\scriptsize $\pm$ 0.15}} & \underline{37.31{\scriptsize $\pm$ 1.96}}   & \underline{25.79{\scriptsize $\pm$ 0.41}} & \underline{14.93 {\scriptsize $\pm$ 3.28}} & \underline{{5.30 \scriptsize $\pm$ 0.18}} \\
\bottomrule
\end{tabular}%
}
\label{tab:semi}
\end{table*}

In our effort to conduct a fair comparison between our proposed method and existing methodologies, we meticulously designed our experiments building upon previous scholarly work. TorchSSL \citep{zhang2021flexmatch}, a sophisticated codebase encompassing a wide array of semi-supervised learning techniques as well as supervised learning implementations, was employed as our foundational code base. This enables us to implement our algorithm effectively and assess its performance on well-established datasets like CIFAR-10, CIFAR-100, and STL-10. In the realm of supervised learning, our unique loss components are applied to annotated data, facilitating the computation of both mutual information loss and information entropy difference loss. For semi-supervised learning scenarios, these loss components are extended to unlabeled data, enhancing the calculation of these loss metrics during the unsupervised learning phase. We use an SGD optimizer, configured with a momentum of 0.9 and a weight decay parameter of $5e^{-4}$. The learning rate was initially set at 0.03, subject to cosine annealing. We report the performance metrics over several runs of seeds. The batch size are maintained at 64 across a comprehensive 1,048,000 iterations training regimen. Concerning model architecture, WideResNet-28-2, WideResNet-28-8, and WideResNet-37-2 are respectively chosen for datasets CIFAR-10, CIFAR-100, and STL-10, 

\begin{table}[t]
\centering
\caption{Results for fully supervised learning}
\begin{tabular}{c|cc}
\toprule
Datasets         & CIFAR-10 & CIFAR-100 \\ \midrule
Fully supervised &          95.35&           80.77\\
Ours (MIR)       &          95.52&           80.81\\
Ours (HDR)       &          \textbf{95.57}&          \textbf{80.96}\\ \bottomrule

\end{tabular}
\label{tab:fullysupervised}
\end{table}

We train supervised and semi-supervised learning models using mutual information and information entropy difference as constraints in the loss function. Table \ref{tab:semi}  and  \ref{tab:fullysupervised} present the performance of semi-supervised and supervised learning, respectively. It is observed that applying mutual information and information entropy constraints led to a slight improvement in supervised learning performance. We believe this is because sufficient labeled data provides adequate information constraints, leading to only a modest enhancement in performance. However, in semi-supervised learning, in most settings, maximizing mutual information and minimizing information entropy resulted in the best or second-best performance. Additionally, our method consistently outperformed our baseline, FreeMatch, across various settings. This suggests that in situations with insufficient labeled samples, additional information constraints can more effectively improve model performance.

\section{Conclusion}

In conclusion, we have made significant advancements in understanding the dynamics of supervised learning by utilizing matrix information and Neural Collapse theory. Our introduction of matrix mutual information ratio (MIR) and matrix entropy difference ratio (HDR) provide novel insights into the interplay between data representations and classification head vectors, serving as new tools to understand the dynamics of neural networks.

Through a series of rigorous theoretical and empirical analyses, we demonstrate the effectiveness of MIR and HDR in elucidating various neural network phenomena, such as grokking, and their utility in improving training dynamics. The incorporation of these metrics as loss functions in supervised and semi-supervised learning shows promising results, indicating their potential to enhance model performance and training efficiency. This study not only contributes to the field of machine learning by offering new analytical tools but also 
applies matrix information to optimize supervised learning algorithms.

\section*{Acknowledgment}
Huimin Ma and Bochao Zou are supported by the National Nature Science Foundation of China (No.U20B2062, No.62227801, No.62206015) and National Science and Technology Major Project (2022ZD0117900).

Weiran Huang is supported by 2023 CCF-Baidu Open Fund and Microsoft Research Asia.

We would also like to express our sincere gratitude to the reviewers of ICML 2024 for their insightful and constructive feedback. Their valuable comments have greatly contributed to improving the quality of our work.

\section*{Impact Statement}
This paper presents work whose goal is to advance the field of Machine Learning. There are many potential societal consequences of our work, none which we feel must be specifically highlighted here.

{\small
\bibliography{reference}
\bibliographystyle{icml2024}
}

\clearpage
\appendix
\onecolumn

\begin{center}
\textbf{\Large Appendix}

\end{center}

\section{Detailed proofs for Neural Collapse related Theorems}

\begin{theorem} \label{direct NC Appendix}
Suppose Neural collapse happens. Then $\operatorname{HDR}(\mathbf{G}(\mathbf{W}^T), \mathbf{G}(\mathbf{M})) = 0$ and $\operatorname{MIR}(\mathbf{G}(\mathbf{W}^T), \mathbf{G}(\mathbf{M})) = \frac{1}{C-1} + \frac{(C-2)\log(C-2)}{(C-1)\log(C-1)}$.
\end{theorem}

\begin{proof} \label{boundary appendix}

By (NC 3), we know that $\mathbf{W}^T =  \frac{\| \mathbf{W}\|_F}{\| \mathbf{M} \|_F} \mathbf{M}$. Noticing that $\frac{\| \mathbf{W}\|_F}{\| \mathbf{M} \|_F} > 0$, we know that $\frac{w_i}{\| w_i \|} = \frac{\tilde{\mu}_i}{ \| \tilde{\mu}_i |}$. It is then very clear that $\mathbf{G}(\mathbf{W}^T) =  \mathbf{G}(\mathbf{M})$. Therefore from definition \ref{gram} and \ref{HDR}, it is clear that $\operatorname{HDR}(\mathbf{G}(\mathbf{W}^T), \mathbf{G}(\mathbf{M})) = 0$. 

Define $\mathcal{E}(\alpha) = \begin{bmatrix}
 1 & \alpha & \cdots & \alpha \\   
  \alpha & 1 & \cdots & \alpha \\  
 \vdots & \vdots & \vdots & \vdots \\
 \alpha & \alpha & \cdots & 1 \\
\end{bmatrix}$. From (NC 2), we know that $\mathbf{G}(\mathbf{W}^T) =  \mathbf{G}(\mathbf{M}) = \mathcal{E}(\frac{-1}{C-1})$ and $\mathbf{G}(\mathbf{W}^T) \odot \mathbf{G}(\mathbf{M}) = \mathcal{E}(\frac{1}{(C-1)^2})$. Notice that $\mathcal{E}(\alpha) = (1-\alpha) \mathbf{I}_C + \alpha \mathbf{1}^T_C \mathbf{1}_C$, we can obtain its spectrum as $1-\alpha$ ($C-1$ times) and $1+ (C-1)\alpha$ ($1$ time). Therefore, we can obtain that $\operatorname{H}(\mathbf{G}(\mathbf{W}^T)) = \operatorname{H}(\mathbf{G}(\mathbf{M})) = \log (C-1)$. And $\operatorname{H}(\mathbf{G}(\mathbf{W}^T) \odot \mathbf{G}(\mathbf{M})) = - \frac{1}{C-1} \log \frac{1}{C-1} - (C-1) \frac{C-2}{(C-1)^2} \log \frac{C-2}{(C-1)^2} = \frac{1}{C-1} \log (C-1) - \frac{C-2}{C-1} \log (C-2) + \frac{2(C-2)}{C-1} \log (C-1) = (2-\frac{1}{C-1})\log (C-1) - \frac{C-2}{C-1} \log (C-2).$ Then then conclusion follows from definition \ref{MIR}.

\end{proof}

As the linear weight matrix $\mathbf{W}$ can be seen as (prototype) embedding for each class. It is natural to consider the mutual information and entropy difference between sample embedding and label embedding. We discuss this in the following theorem \ref{feature NC}.

\begin{corollary} \label{feature NC appendix}
Suppose the dataset is class-balanced, $\mu_G =0$ and Neural collapse happens. Denote $\mathbf{Z}_1 = [h(\mathbf{x}_1) \cdots h(\mathbf{x}_n)] \in \mathbb{R}^{d \times n}$ and $\mathbf{Z}_2 = [w_{y_1} \cdots w_{y_n}] \in \mathbb{R}^{d \times n}$. Then $\operatorname{HDR}(\mathbf{Z}_1, \mathbf{Z}_2) = 0$ and $\operatorname{MIR}(\mathbf{Z}_1, \mathbf{Z}_2) = \frac{1}{C-1} + \frac{(C-2)\log(C-2)}{(C-1)\log(C-1)}$.
\end{corollary}

\begin{proof}
Denote $n_1 = \frac{n}{C}$ the number of samples in each class. Without loss of generality, assume samples are arranged as $C$ consecutive groups, each group has samples from the same class. 

Define $\mathcal{E}(\alpha) = \begin{bmatrix}
 1 & \alpha & \cdots & \alpha \\   
  \alpha & 1 & \cdots & \alpha \\  
 \vdots & \vdots & \vdots & \vdots \\
 \alpha & \alpha & \cdots & 1 \\
\end{bmatrix}$ and $\mathbf{1}_{n_1, n_1}$ a $n_1 \times n_1$ matrix with all its element $1$.

From (NC 1), we know that $\mathbf{G}(\mathbf{W}^T) =  \mathbf{G}(\mathbf{M}) = \mathcal{E}(\frac{-1}{C-1}) \otimes \mathbf{1}_{n_1, n_1}$ and $\mathbf{G}(\mathbf{W}^T) \odot \mathbf{G}(\mathbf{M}) = \mathcal{E}(\frac{1}{(C-1)^2}) \otimes \mathbf{1}_{n_1, n_1}$, where $\otimes$ is the Kronecker product. By the property of the Kronecker product spectrum, the non-zero spectrum of $\mathbf{G}(\mathbf{W}^T)$ will be $n_1$ times the spectrum of $ \mathcal{E}(\frac{-1}{C-1}) $. Then this corollary follows by the same proof of the theorem \ref{direct NC}.
\end{proof}

\section{Some theoretical guarantees for HDR}

Mutual information is a very intuitive quantity in information theory. On the other hand, it seems weird to consider the difference of entropy, but we will show that this quantity is closely linked with comparing the approximation ability of different representations on the same target.

For ease of theoretical analysis, in this section, we consider the MSE regression loss.

The following lemma \ref{approx} shows that the regression of two sets of representations $\mathbf{Z}_1$ and $\mathbf{Z}_2$ to the same target $\mathbf{Y}$ are closely related. And the two approximation errors are closely related to the regression error of $\mathbf{Z}_1$ to $\mathbf{Z}_2$.

\begin{lemma} \label{approx Appendix}
Suppose $\mathbf{W}^*_1, \mathbf{b}^*_1 = \argmin_{\mathbf{W}, \mathbf{b}}  \|\mathbf{Y} - (\mathbf{W} \mathbf{Z}_1 + \mathbf{b} \mathbf{1}_N )\|_F$. Then $\min_{\mathbf{W}, \mathbf{b}} \|\mathbf{Y} - (\mathbf{W} \mathbf{Z}_2 + \mathbf{b} \mathbf{1}_N )\|_F \leq \min_{\mathbf{W}, \mathbf{b}}  \|\mathbf{Y} - (\mathbf{W} \mathbf{Z}_1 + \mathbf{b} \mathbf{1}_N )\|_F + \| \mathbf{W}^*_1\|_F \min_{\mathbf{H}, \mathbf{\eta}}  \|\mathbf{Z}_1 - (\mathbf{H} \mathbf{Z}_2 + \mathbf{\eta} \mathbf{1}_N )\|_F$.   
\end{lemma}

\begin{proof}
Suppose $\mathbf{H}^*, \mathbf{\eta}^* = \argmin_{\mathbf{H}, \mathbf{\eta}}  \|\mathbf{Z}_1 - (\mathbf{H} \mathbf{Z}_2 + \mathbf{\eta} \mathbf{1}_N )\|_F$. Then $\min_{\mathbf{W}, \mathbf{b}}  \|\mathbf{Y} - (\mathbf{W} \mathbf{Z}_2 + \mathbf{b} \mathbf{1}_N )\|_F \leq  \|\mathbf{Y} - (\mathbf{W}^*_1\mathbf{H}^* \mathbf{Z}_2 + (\mathbf{b}^*_1+\mathbf{W}^*_1 \eta^*) \mathbf{1}_N )\|_F \leq \|\mathbf{Y} - (\mathbf{W}^*_1 \mathbf{Z}_1 + \mathbf{b}^*_1 \mathbf{1}_N ) \|_F + \|\mathbf{W}^*_1(\mathbf{Z}_1 - (\mathbf{H}^* \mathbf{Z}_2 + \mathbf{\eta}^* \mathbf{1}_N )) \|_F \leq  \|\mathbf{Y} - (\mathbf{W}^*_1\mathbf{H}^* \mathbf{Z}_2 + (\mathbf{b}^*_1+\mathbf{W}^*_1 \eta^*) \mathbf{1}_N )\|_F \leq \|\mathbf{Y} - (\mathbf{W}^*_1 \mathbf{Z}_1 + \mathbf{b}^*_1 \mathbf{1}_N ) \|_F + \|\mathbf{W}^*_1\|_F \|\mathbf{Z}_1 - (\mathbf{H}^* \mathbf{Z}_2 + \mathbf{\eta}^* \mathbf{1}_N ) \|_F$.     
\end{proof}

From lemma \ref{approx}, we know that the regression error of $\mathbf{Z}_1$ to $\mathbf{Z}_2$ is crucial for understanding the differences of representations. We further bound the regression error with rank and singular values in the following lemma \ref{rank and singuar}.

\begin{lemma} \label{rank and singuar Appendix}
Suppose $\mathbf{Z}_1 = [\mathbf{z}^{(1)}_1 \cdots \mathbf{z}^{(1)}_N] \in \mathbb{R}^{d{'} \times N}$ and $\mathbf{Z}_2 = [\mathbf{z}^{(2)}_1 \cdots \mathbf{z}^{(2)}_N] \in \mathbb{R}^{d \times N}$ and $\text{rank}(\mathbf{Z}_1) > \text{rank}(\mathbf{Z}_2)$. Denote the singular value of $\frac{\mathbf{Z}_1}{\sqrt{N}}$ as $\sigma_1 \geq \cdots \geq \sigma_{N}$. Then $\min_{\mathbf{H}, \mathbf{\eta}} \frac{1}{N} \|\mathbf{Z}_1 - (\mathbf{H} \mathbf{Z}_2 + \mathbf{\eta} \mathbf{1}_N )\|^2_F \geq \sum^{\text{rank}(\mathbf{Z}_1)}_{j= \text{rank}(\mathbf{Z}_2)+2} (\sigma_j)^2$.  
\end{lemma}

\begin{proof}
The proof idea is similar to \citep{garrido2023rankme}. Suppose $\mathbf{H}^*, \mathbf{\eta}^* = \argmin_{\mathbf{H}, \mathbf{\eta}} \frac{1}{N} \|\mathbf{Z}_1 - (\mathbf{H} \mathbf{Z}_2 + \mathbf{\eta} \mathbf{1}_N )\|^2_F$ and $r = \text{rank}(\mathbf{H}^* \mathbf{Z}_2 + \mathbf{\eta}^* \mathbf{1}_N )$.

Then from Eckart–Young–Mirsky theorem $\frac{1}{N} \|\mathbf{Z}_1 - (\mathbf{H}^* \mathbf{Z}_2 + \mathbf{\eta}^* \mathbf{1}_N )\|^2_F \geq \sum^{N}_{j= r+1} (\sigma^{(1)}_j)^2$. And note $r \leq \text{rank}(\mathbf{Z}_2)+1$ and singular value index bigger than rank is $0$. The conclusion follows.
\end{proof}

The bound given by lemma \ref{rank and singuar} is not that straightforward to understand. Assuming the features are normalized, we successfully derived the connection of regression error and ratio of ranks in theorem \ref{rank ratio}.

\begin{theorem} \label{rank ratio Appendix}
Suppose $\| \mathbf{z}^{(1)}_j \|_2=1$, where ($1 \leq j \leq N$). Then lower bound of approximation error can be upper-bounded as follows:
$\sum^{\text{rank}(\mathbf{Z}_1)}_{j= \text{rank}(\mathbf{Z}_2)+2} (\sigma_j)^2 \leq \frac{\text{rank}(\mathbf{Z}_1)-\text{rank}(\mathbf{Z}_2)-1}{\text{rank}(\mathbf{Z}_1)} \leq 1-\frac{\text{rank}(\mathbf{Z}_2)}{\text{rank}(\mathbf{Z}_1)}$.
\end{theorem}

\begin{proof}
The proof is direct by noticing the summation of the square of singular values is $1$ and we have already ranked singular values by their indexes.    
\end{proof}

\end{document}